\documentclass{article}

\newcommand{\additionalpackages}{\setlength{\oddsidemargin}{0.25 in}
\setlength{\evensidemargin}{-0.25 in}
\setlength{\topmargin}{-0.6 in}
\setlength{\textwidth}{6 in}
\setlength{\textheight}{8.5 in}
\setlength{\headsep}{0.75 in}
\setlength{\parindent}{0 in}
\setlength{\parskip}{0.1 in}}
\newcommand{\beforeappendix}{\section*{Acknowledgements}
This project is supported by the European Research Council (ERC) under the European Union’s Horizon 2020 research and innovation program (grant agreement No. 882396), by the Israel Science Foundation and the Yandex Initiative for Machine Learning at Tel Aviv University and by a grant from the Tel Aviv University Center for AI and Data Science (TAD).
OL is also supported by the Google PhD fellowship award (2025).
}
\newcommand{\afterappendix}{}

\usepackage[numbers]{natbib}
\additionalpackages

\usepackage[utf8]{inputenc} % allow utf-8 input
\usepackage[T1]{fontenc}    % use 8-bit T1 fonts
\usepackage{hyperref}       % hyperlinks, MUST be before cleveref
\usepackage{url}            % simple URL typesetting
\usepackage{booktabs}       % professional-quality tables
\usepackage{amsfonts}       % blackboard math symbols
\usepackage{nicefrac}       % compact symbols for 1/2, etc.
\usepackage{microtype}      % microtypography

\usepackage{graphicx} % Required for inserting images
\usepackage{amsmath}
\usepackage{amssymb}
\usepackage{mathtools}
\usepackage{amsthm}
\usepackage{algorithm}
\usepackage{algpseudocode}
\usepackage{xcolor}
\usepackage{tikz}
\usepackage{cleveref}
\usepackage{dsfont}
\usepackage{enumitem}
\usepackage{bbm}
\usepackage{delimset}
\usetikzlibrary{positioning}

\theoremstyle{plain}
\newtheorem{theorem}{Theorem}[section]
\newtheorem{proposition}[theorem]{Proposition}
\newtheorem{lemma}[theorem]{Lemma}
\newtheorem{corollary}[theorem]{Corollary}
\theoremstyle{definition}
\newtheorem{definition}[theorem]{Definition}

\theoremstyle{remark}
\newtheorem{remark}[theorem]{Remark}

\makeatletter \newcommand{\ForceCrefTypeInEnv}[2]{%
\AddToHook{env/#1/begin}{% 
\let\Cref@oldlabel\label \def\label##1{\Cref@oldlabel[#2]{##1}}% 
}% 
\AddToHook{env/#1/end}{\let\label\Cref@oldlabel}} 
\makeatother
\ForceCrefTypeInEnv{lemma}{lemma} 
\ForceCrefTypeInEnv{definition}{definition} 
\ForceCrefTypeInEnv{proposition}{proposition} 
\ForceCrefTypeInEnv{claim}{claim} 
\ForceCrefTypeInEnv{corollary}{corollary} 
\ForceCrefTypeInEnv{example}{example}
\ForceCrefTypeInEnv{remark}{remark}
\Crefname{lemma}{Lemma}{Lemmas}
\Crefname{definition}{Definition}{Definitions} 
\Crefname{proposition}{Proposition}{Propositions} 
\Crefname{claim}{Claim}{Claims}
\Crefname{corollary}{Corollary}{Corollaries} 
\Crefname{example}{Example}{Examples}
\Crefname{remark}{Remark}{Remark}

\renewcommand{\citet}[1]{\citeauthor{#1} (\citeyear{#1})}

\DeclareMathOperator*{\argmax}{arg\,max}
\DeclarePairedDelimiter\floor{\lfloor}{\rfloor}

\newcommand{\norminf}[1]{\norm{#1}_\infty}
\newcommand{\normone}[1]{\norm{#1}_1}

\newcommand{\prob}[1]{\mathbb{P} \left( #1 \right)}
\newcommand{\E}{\mathbb{E}}
\newcommand{\Expect}[1]{\E\left(#1\right)}

\newcommand{\R}{\mathbb{R}}

\newcommand{\calS}{\mathcal{S}}
\newcommand{\calA}{\mathcal{A}}
\newcommand{\F}{\mathcal{F}}

\newcommand{\agentsNumberEpsilon}{144\frac{S^5 H^6 A (\log(1/\delta')+2S)}{\epsilon^2}}
\newcommand{\agentsNumberEpsilonMainText}{O\left(\frac{S^5 H^6 A (\log(SHA/\delta)+S)}{\epsilon^2}\right)}

\newcommand{\MDP}{\mathcal{M}}
\newcommand{\realP}{P}

\newcommand{\hpi}{\widehat{\pi}}

\newcommand{\hP}{\widehat{P}}
\newcommand{\hPpi}{\widehat{P}^\pi}
\newcommand{\hq}{\widehat{q}}

\newcommand{\tS}{\widehat{\mathcal{S}}^\beta}

\newcommand{\betaEstSymbol}{1}

\newcommand{\tP}{P^{\betaEstSymbol}}
\newcommand{\tq}{q^{\betaEstSymbol}}
\newcommand{\tPpolicy}[1]{P^{\betaEstSymbol,#1}}

\newcommand{\morebeta}{{2\beta}}
\newcommand{\morebetaMDPSymbol}{2}

\newcommand{\bP}{P^\morebetaMDPSymbol}
\newcommand{\bq}{q^\morebetaMDPSymbol}
\newcommand{\bS}{\mathcal{S}^\morebetaMDPSymbol}
\newcommand{\bPpolicy}[1]{P^{\morebetaMDPSymbol,#1}}

\newcommand{\calK}{\mathcal{K}}

\newcommand{\starState}{s^\star}

\newcommand{\keyReward}{r^{\mathrm{key}}}
\newcommand{\starAction}{a^\star}

\newcommand{\learningAlgorithm}{\texttt{A}}
\newcommand{\goodAgents}{\mathcal{G}}
\newcommand{\randomMdpDist}{\mathcal{E}}

\newcommand{\valueAlg}[2]{V^{\hpi}(s_0;#1,#2)}

\newcommand{\noAgentsLast}{\abs{\goodAgents_{H-1}}=0}

\newcommand{\keyDynamics}{P^{\mathrm{key}}}
\newcommand{\keyDynamicsEnum}[1]{P^{\mathrm{key,#1}}}

\newcommand{\noAgentsHtag}{\abs{\goodAgents_{H'}} = 0}
\newcommand{\interaction}{\mathcal{I}_{\learningAlgorithm}}

\newcommand{\hBestArrivingPolicyGeneral}[2]{\widehat{\pi}^{#1,#2}}
\newcommand{\hBestArrivingPolicy}{\hBestArrivingPolicyGeneral{h}{s}}
\newcommand{\one}{\mathbbm{1}}
\newcommand{\sinkState}{s_{\mathrm{sink}}}
\newcommand{\sinit}{s_{0}}

\newcommand{\algname}{\textsc{H-Phase Multi-Agent Reward-Free Exploration}}
\newcommand{\algacronyms}{\texttt{H-MARFE}}
\title{The Horizon Threshold in Cooperative Multi-Agent Reward-Free Exploration}

\author{%
  Idan Barnea \\
  Tel Aviv University \\
  \texttt{idan1id@gmail.com} \\
  \and
  Orin Levy \\
  Tel Aviv University \\
  \texttt{orinlevy@mail.tau.ac.il} \\
  \and
  Yishay Mansour \\
  Tel Aviv University \& Google Research \\
  \texttt{mansour.yishay@gmail.com}
}
\date{}

\begin{document}

\maketitle

\begin{abstract}
We study cooperative multi-agent reinforcement learning in the setting of reward-free exploration, where multiple agents jointly explore an unknown MDP in order to learn its dynamics (without observing rewards). We focus on a tabular finite-horizon MDP and adopt a phased learning framework. In each learning phase, multiple agents independently interact with the environment. More specifically, in each learning phase, each agent is assigned a policy, executes it, and observes the resulting trajectory. Our primary goal is to characterize the tradeoff between the number of learning phases and the number of agents, especially when the number of learning phases is small.

Our results identify a regime change governed by the horizon $H$. When the number of learning phases equals $H$, we present a computationally efficient algorithm that uses only $\tilde{O}(S^6 H^6 A / \epsilon^2)$ agents to obtain an $\epsilon$ approximation of the dynamics (i.e., yields an $\epsilon$-optimal policy for any reward function). We complement our algorithm with a lower bound showing that any algorithm restricted to $\rho < H$ phases requires at least $A^{H/\rho}$ agents to achieve constant accuracy. 
Thus, we show that having $\Theta(H)$ learning phases is both necessary and sufficient when restricting the number of agents to be polynomial.

\end{abstract}

\section{Introduction}
Modern learning systems increasingly rely on multiple agents that cooperate to explore and learn about an unknown dynamic environment.
Such cooperative multi-agent settings arise naturally in robotics and distributed control 
\cite{robotics-survey-2025-orr},
%\cite{franchi2009sensor,yuan2010cooperative,ma2007genetic},
as well as in large language model (LLM) fine-tuning, where a central challenge lies in constructing a high-quality dataset.
This dataset is typically collected once and subsequently reused to train models across a wide range of downstream objectives (see, e.g.,~\cite{DBLP:journals/jmlr/ChungHLZTFL00BW24,DBLP:journals/tmlr/KaufmannWBH25-RLHF}).

The above scenarios motivate the study of \emph{cooperative multi-agent reinforcement learning (RL)}~\cite{lancewicki2022-coop-mdp} (sometimes called concurrent RL \cite{DBLP:conf/icml/DimakopoulouR18-concurrent-rl}), in which multiple agents interact with an unknown environment in parallel.
This literature extends the classical RL framework based on tabular finite-horizon Markov decision processes (MDPs) to settings where several agents operate on independent copies of the same MDP and share the information they collect.

Finite-horizon MDPs are defined by finite state and action spaces, initial state, transition dynamics and reward function. The action selection strategy of an agent is referred to as \emph{policy}.
Each single-agent episode is defined by $H$ interaction steps with the MDP, described by a single \emph{trajectory}. At each step, every agent observes the current state it reaches and selects an action according to its own policy, which may differ across agents.
The agent then observes a reward associated with the current state and action and transitions to a next state.
All agents experience rewards and transitions governed by the same reward function and transition dynamics of the shared MDP.
The learning algorithm has access to the trajectories generated by all agents and can leverage this aggregated experience to update the agents’ policies.

However, in many realistic settings as described earlier, the reward function is not specified during data collection, or multiple reward functions may later be of interest for the same environment.
This naturally motivates adopting the \emph{reward-free exploration} setup~\cite{DBLP:conf/icml/JinKSY20}.

In reward-free exploration, the goal is to learn the transition dynamics of the environment itself. Rather than optimizing the environment's particular cumulative reward collected during an episode, the learner ignores rewards and focuses on collecting data that enables an accurate estimation of the underlying dynamics.

Crucially, the learned model should be accurate enough so that, for any reward function revealed after data collection, it can be used to reliably evaluate and optimize policies. In particular, planning with the learned dynamics should yield an accurate approximation of the optimal expected cumulative reward, i.e., the \emph{value} of the best policy under the true dynamics of the environment and the given reward function.
Since interaction with the environment is often costly, a central objective in reward-free exploration is to achieve this guarantee using as few trajectories as possible.

In this work, we study the  novel setting of \emph{cooperative multi-agent reward-free exploration};
In which,  multiple agents explore the environment in parallel under the coordination of a centralized reward-free exploration algorithm, with the shared objective of efficiently learning the underlying dynamics.
To model parallel interaction, we gather parallel episodes into 
\emph{learning phases}, where a learning phase consists of a round of simultaneous $m$ single-agent episodes in which each of the $m$ agents interacts with an independent environment and collects a trajectory without observing the rewards.
The centralized learning  algorithm aggregates the data collected by all  agents in each phase to update its estimate of the environment dynamics and the agents’ exploration policies in the next phase.

In this setting, learning efficiency is governed by two key resources.
The first is the \emph{parallel time}, measured by the number of learning phases, which determines the distributed computation time required for data collection.
% The first is the \emph{parallel time}, measured by the number of learning phases, which captures the wall-clock time required for data collection.
The second is the \emph{agent complexity}, measured by the number of agents interacting with the environment in each learning phase, which captures the degree of parallelism available to the learning algorithm.
These two resources are inherently coupled, giving rise to a fundamental trade-off between parallel time and agent complexity.
At one extreme, when parallel time is unlimited, reward-free exploration can be done sequentially using a single agent (recovering the classical reward-free literature, e.g., \cite{DBLP:conf/icml/JinKSY20,DBLP:conf/icml/MenardDJKLV21,DBLP:conf/alt/KaufmannMDJLV21}).
At the other extreme, aggressively limiting parallel time, e.g., to a constant number of learning phases, requires deploying an exponential number of agents (as shown by our lower bound, presented in the sequel).

% In this work, we study the fundamental question of \emph{how to balance parallel time and agent complexity to achieve reward-free guarantees as efficiently as possible}, characterizing how parallelism can accelerate exploration and identifying regimes with a favorable trade-off between these two resources.

\textbf{Our contributions.}
We study cooperative multi-agent reward-free exploration in tabular Markov decision processes (MDPs) with $S$ states and $A$ actions, where each episode has a finite horizon of length $H$.
Our central focus is understanding the trade-off between \emph{parallel time}, measured by the number of learning phases $\rho$, and \emph{agent complexity}, measured by the number of agents $m$ in each learning phase.

Our results reveal a \emph{regime change in this trade-off at the horizon scale $H$.}
Specifically, when $\rho = H$, reward-free exploration can be achieved efficiently using a polynomial number of agents.
In contrast, when the number of learning phases is reduced below the horizon, i.e., $\rho < H $, any reward-free algorithm necessarily requires the number of agents to grow exponentially in $H / \rho$.
Concretely, our main technical contributions are:

\begin{itemize}
    \item \textbf{Upper bound.}
        We present \algacronyms{}, a computationally efficient cooperative multi-agent reward-free exploration algorithm that learns the unknown dynamics in exactly $\rho = H$ phases.
        We prove that, with $m = \tilde{O}(S^6 H^6 A / \epsilon^2)$ agents deployed per phase, the learned model yields an $\epsilon$-optimal policy for any given reward function.

    \item \textbf{Lower bound.}
    We prove that reducing the number of learning phases significantly below the horizon fundamentally limits parallel efficiency.
    In particular, any algorithm that uses $\rho < H$ learning phases must deploy at least
    $m = \Omega\left(A^{H/\rho} / \rho\right)$ agents.
    This lower bound shows that an order of $H$ learning phases is essential to avoid exponential agent complexity.
\end{itemize}

Our analysis introduces a novel technique: layer-wise exploration that scales estimation accuracy with reachability, departing from standard regret minimization approaches. Unlike classical regret minimization algorithms (or bonus-based algorithms), which are inherently sequential, our approach is parallelizable. It guarantees low additive error in the occupancy measure difference, which simultaneously ensures accurate value estimation for any reward function and enables the computation of effective exploration directly from the estimated model.

\subsection{Related Work}

% There are several related lines of works, listed below.

\textbf{Reward-free.}
The reward-free model \cite{DBLP:conf/iclr/EysenbachGIL19-reward-free-first} was formally formulated by \citet{DBLP:conf/icml/JinKSY20}, who established a sub-optimal sample complexity bound. This bound was subsequently tightened \cite{DBLP:conf/alt/KaufmannMDJLV21,DBLP:conf/icml/MenardDJKLV21}.
These works are inherently sequential\footnote{\citet{DBLP:conf/icml/JinKSY20} rely on regret-based algorithms to construct exploration policies, and regret guarantees fundamentally apply only in sequential settings; \cite{DBLP:conf/alt/KaufmannMDJLV21,DBLP:conf/icml/MenardDJKLV21} use bonuses to guarantee exploration.},
and it is unclear how their proposed techniques can be parallelized.
\citet{DBLP:conf/nips/QianHS24CmdpSimchiLevi} study reward-free exploration for contextual MDPs under efficient oracle access to a finite class of dynamics.
Applying their algorithm to our setting would require access to an infinite class containing all possible dynamics, which is clearly unlearnable.
Reward-free exploration for linear MDPs has been studied \cite{DBLP:conf/nips/WangDYS20-reward-free-linear-function-approx,DBLP:conf/nips/ZhangZG21-reward-free-linear-function-approx,DBLP:conf/icml/WagenmakerCSDJ22a-reward-free-linear-mdp, DBLP:conf/iclr/HuCH23-reward-free-linear-mdp}, but these works are sequential in nature and cannot be easily parallelized.
Additional work addressing reward-free exploration include 
Block-MDP and low-rank settings
\cite{DBLP:conf/nips/AgarwalKKS20-reward-free-low-rank,DBLP:conf/icml/ZhangSUWAS22-reward-free-block-mdp,DBLP:conf/iclr/ChengHL023-reward-free-low-rank,DBLP:conf/nips/MhammediBFR23-reward-free-low-rank,DBLP:conf/icml/AmortilaFK24-reward-free-block-low-rank-and-l1-coverage}, 
and other assumptions 
\cite{pmlr-v139-zhang21e-reward-agnostic,DBLP:conf/nips/Chen0K0A22-reward-free-function-approx,DBLP:conf/icml/MiryoosefiJ22-reward-free-constraints}.

\textbf{Cooperative Multi-Agent RL.}
In the cooperative multi-agent setting \cite{lancewicki2022-coop-mdp}—also known as concurrent RL \cite{DBLP:conf/icml/DimakopoulouR18-concurrent-rl,DBLP:conf/nips/DimakopoulouOR18-concurrent-rl-scalable}—multiple agents interact with the environment simultaneously.
This setting is also referred to as multi-batch or deployment-efficient RL, and is closely related to the low-switching setting, as the number of policy switches corresponds to the number of parallel phases.
Cooperative multi-agent was studied for regret minimization in the tabular setting \cite{lancewicki2022-coop-mdp}, for linear MDPs \cite{DBLP:conf/l4dc/HsuP25Safe}, and for linear-quadratic dynamical systems \cite{DBLP:conf/uai/Asghari0N20}.
% In this setting all agents run for $K$ phases and the goal is to minimize the sum of their regrets.
Since regret minimization algorithms (and also bonus-based algorithms) are designed for sequential decision-making, their learning-phase complexity scales linearly with the number of episodes.

Prior empirical works on cooperative multi-agent RL (e.g., \citep{DBLP:conf/emnlp/0005ZC24,DBLP:journals/corr/HeessTSLMWTEWER17-dppo,DBLP:conf/icml/DimakopoulouR18-concurrent-rl,DBLP:conf/nips/DimakopoulouOR18-concurrent-rl-scalable,DBLP:journals/tnn/HaoYTBLMLW24-multi-agent-exploration-survey}) focus on empirical evaluation of reward-based algorithms and not on theoretical analysis.
Beyond the distinction between theory and empirical evaluation, direct comparison is further complicated by the fact that our work targets reward-free exploration.
For instance, the A3C algorithm \cite{DBLP:conf/emnlp/0005ZC24} relies on sufficient inherent randomization in the rewards and dynamics of the MDP. This allows them to significantly simplify the exploration perspective. Note that this assumption that breaks down in deterministic MDPs.
In our lower-bound construction with a sparse reward, A3C requires exponentially many ($2^H$) learning phases before observing any reward, whereas our algorithm learns the dynamics in only $H$ phases with a polynomial number of agents.
A similar issue arise in \emph{seed sampling} \cite{DBLP:conf/icml/DimakopoulouR18-concurrent-rl} and DPPO \cite{DBLP:journals/corr/HeessTSLMWTEWER17-dppo}.

\citet{DBLP:conf/icml/QiaoYM022loglogt} guarantee reward-free learning with $O(SHA)$ switches and $\tilde{O}(H^5S^2A/\epsilon^2)$ sample complexity, but their algorithm is computationally inefficient.
In contrast, our algorithm is efficient and requires only $H$ learning phases (where switches and phases are interchangeable notions).
Prior work has also addressed regret minimization in the cooperative multi-agent setting with a focus on high concurrency
\citep{DBLP:conf/nips/Zhang0J20ModelFreeParallel,DBLP:conf/nips/ZhangJZJ22Batch,DBLP:conf/nips/ZhaoHG24},
as well as linear MDPs with deployment costs
\cite{DBLP:conf/iclr/HuangCZQJL22Deployment,zhang2025reward-free-deployment-linear-func-approx}, where the latter achieving deployment cost $H$ and sample complexity $O(d^{9} H^{15} / \epsilon^{5})$.

\textbf{Additional related work} characterizes horizon effects, mostly in the context of off-policy evaluation \citep{DBLP:journals/jmlr/RossB10-horizon,DBLP:conf/nips/LiuLTZ18-horizon,DBLP:conf/colt/ZhangJD21-horizon,DBLP:journals/ior/KallusU22-horizon,DBLP:conf/nips/FosterBM24-horizon}.

\section{Preliminaries}
\label{sec:preliminaries}

We study \emph{reward-free exploration in a cooperative multi-agent reinforcement learning setting}.
We begin by defining the single-agent MDP and the notion of reward-free exploration, which will later be extended to the multi-agent case.

\textbf{Markov Decision Process.}
An episodic finite-horizon \emph{Markov Decision Process} (MDP) $\MDP$ is formally defined as a tuple
$\MDP = (\calS, \sinit, \calA, H, \realP, r)$
where $\calS$ is a finite state space with $S$ states, $\sinit \in \calS$ is the initial state\footnote{Without loss of generality, we assume an initial-state distribution at the first timestep.}, $\calA$ is a finite action space contains $A$ actions, $H$ is the episode horizon length,
%where $h\in\{0,\ldots, H-1\}$ are the timesteps,
$\realP = \{\realP_h\}_{h=0}^{H-1}$ is the transition dynamics, and $r = \{r_h\}_{h=0}^{H-1}$ is the reward function.

The transition dynamics $\realP$ are stochastic and unknown.
They specify the transition probabilities at each timestep $h \in [H-1]$ (where $[H-1] := \{0,\ldots,H-1\}$).
At each timestep $h$, for a given state $s$ and action $a$, the next state $s'$ is drawn according to
$\realP_h(s' \mid s, a)$.

The expected reward at timestep $h$ is given by
$r_h(s,a) = \E\left[ R_h(s,a) \right],$
where $R_h(s,a) \in [0,1]$ is a random variable sampled from an unknown reward distribution associated with $(h,s,a)$.

An episode is a single run of interaction with the MDP.
The sequence of states visited, actions chosen and reward observed defines the trajectory observed in the episode, $(s_0, a_0, r_0, s_1, a_1, r_1, \ldots, s_H)$, where $s_h,a_h,r_h$ denote state, action and reward observed at timestep $h$.

\textbf{Policies and values.} 
A (deterministic Markovian) \emph{policy} defines a mapping from timestep $h$ and state $s$ to an action. i.e., 
$\pi : \calS \times [H-1] \rightarrow \calA$.
We denote by $\pi_h(s)=\pi(s,h)$ the action selected by $\pi$ for state $s$ at timestep $h$.
For MDP $\MDP$ with dynamics $\realP$ and reward $r$, the value of a policy $\pi$ from the initial state $s_0$ is given by
$
    V^\pi_{\realP,r}
    \;:=\;
    \E_{\realP,\pi}\left[ \sum_{h=0}^{H-1} r_h(s_h, a_h) \mid s_0 \right].
$
We explicitly index the value function by $(\realP,r)$ to emphasize that our focus is on learning the dynamics.
The \emph{optimal policy} $\pi^\star$ satisfies $\pi^\star \in \argmax_\pi V^\pi_{\realP,r}$. We note that there always exists an optimal deterministic policy \cite{puterman2014markov}. The optimal value for $\realP,r$ is denoted by $V^\star_{\realP,r} := V^{\pi^\star}_{\realP,r}$.
A policy $\pi$ is \emph{$\epsilon$-optimal} if
it holds that $V^\star_{\realP,r} - V^\pi_{\realP,r} \le \epsilon.$
In addition, for any policy $\pi$, the implied transition matrix, denoted $P^\pi_h\in \R^{S\times S}$, is a row stochastic matrix where $P^\pi_h[s,s'] = P_h(s'\mid s, \pi(s))$.

\textbf{Occupancy measures.}
For a policy $\pi$ and dynamics $\realP$, we define the occupancy measure at timestep $h$ as the joint distribution over states and actions induced by $\pi$ under $\realP$.
Formally, 
$q_h(s,a | \pi, \realP) := \prob{ s_h = s, a_h = a | \realP,\pi}$. We also denote $q_h(s | \pi, \realP) := \prob{ s_h = s| \realP,\pi}$, which is the probability to visit a state at timestep $h$. Clearly, $q_h(s,a | \pi, \realP) = q_h(s | \pi, \realP) \prob{[\pi_h(s) = a]}$.
%, where $\one{}$ denotes the indicator.
Occupancy measures capture the visitation frequencies of state-action pairs along an episode and thus define a distribution over them.
%that will be used in our analysis.
% $    V^{\pi}_{\realP,r}
%     =
%     \sum_{h,s,a} q_h(s,a|\pi,\realP) r_h(s,a)$.

\textbf{Reward-free exploration}
is the task of learning an accurate estimate of the dynamics $\realP$, which can later be used to compute an $\epsilon$-optimal policy for any given reward function.
In the reward-free exploration phase, the rewards are not observed by the learners, and the focus is to learn an accurate approximation of the dynamics.
Formally, let $\hP$ denote the estimated dynamics learned by a reward-free algorithm when interacting with an unknown environment $\MDP$.
For any reward function $r$, let
$\hpi_r := \argmax_{\pi} V^{\pi}_{\hP,r}$.
The  goal of the reward-free learning algorithm is to produce an estimate $\hP$ such that, with high probability,
$
    V^\star_{\realP,r} - V^{\hpi_r}_{\realP,r} \le \epsilon.
$

\textbf{Cooperative multi-agent RL.}
We consider a cooperative multi-agent reinforcement learning (MARL, see \cite{Zhang2021MARL}) setting in which multiple agents interact with a shared environment and share information.
We assume \emph{fresh randomness} across agents, i.e., during each learning phase, all random variables
are i.i.d sampled for each agent.
%governing state transitions are independently instantiated for each agent.
In particular, even when multiple agents take identical actions in identical states at the same timestep, they observe independently sampled next state.

A learning algorithm in this setting operates over $\rho$ learning phases, each phase decomposed into $m$ episodes of interactions, one for each agent.
Formally, in each phase $i \in [\rho]$, the algorithm deploys all $m$ agents to interact with the environment in parallel.
Each agent executes a single episode according to its assigned policy, and there is no communication or interaction between agents during the phase.
At the end of each phase, the algorithm aggregates the trajectories collected by all agents and updates the policies used in subsequent phases.
%We denote the number of multi-agent phases as the \emph{parallel time} of the algorithm, and the number of agents in each phase as the \emph{agent complexity}. 

Note we distinguish the agent complexity from sample complexity.
Sample complexity is the total number of environment interactions required by the algorithm, while agent complexity is the required number of agents in each learning phase. They are closely related: multiplying the agent complexity by the number of learning phases yields the sample complexity, but the reverse is subtler, since a sequential algorithm may not exploit parallelism across agents.

\textbf{Cooperative Multi-agent RL for Reward-Free Exploration.}
We focus on the \emph{reward-free exploration} setting, where agents do not observe the rewards of the environment and their goal is to explore the dynamics in parallel to construct an accurate estimate of it.
As in the single-agent reward-free setting, policies are designed to maximize information acquisition rather than cumulative reward.
The overall framework for cooperative multi-agent reward-free exploration is presented in Protocol~\ref{Alg:environment}.

% The learning objective in this setting is to balance parallel time and agent complexity so as to achieve efficient reward-free exploration in cooperative multi-agent environments.

\begin{algorithm}[t]
\floatname{algorithm}{Protocol}
\caption{Interaction Protocol: Cooperative Multi-Agent for Reward-Free Exploration\label{Alg:environment}} 
\begin{algorithmic}[1]
\For{each learning phase: $i = 0, \dots, \rho-1$}
    \For{each agent $j\in[m]$ in parallel}
        \State Run policy $\pi_{j,i}$ and observe the trajectory.
        % horizon $H$ and ignore rewards.
        % \State Gather information from all agents.
    \EndFor
\EndFor
\State Construct estimated dynamics $\hP$, and output it.
\end{algorithmic}
\end{algorithm}

\section{Upper Bound}
Our goal is to accelerate reward-free exploration by leveraging multiple agents to
reduce the number of learning phases, while keeping the total number of agents
polynomial in $S$, $A$, $H$, and $\epsilon^{-2}$. As discussed before, this naturally induces the parallel time versus agent complexity tradeoff. As shown by our lower bound (see \Cref{sec:lower-bound}), at least an order $H$ learning phases is necessary for efficient learning.
Accordingly, we design an algorithm that operates in exactly $H$ learning phases
and requires only a polynomial number of agents.

We note that agent complexity at the scale of $O(1/\epsilon^3)$ can be achieved
using a simple construction. We sketch the idea here.
In phase $i$, we explore the timestep-$i$ dynamics $\realP_i$.
We explore each state–action pair $(s,a)$ by assigning to $S/\epsilon^3$ agents a policy that maximizes the probability of reaching this state–action pair $(s,a)$.
If at least $1/\epsilon^2$ agents reach $(s,a)$, we obtain sufficiently many samples
to accurately estimate their transitions. Otherwise, it means the probability to visit them is lower than $\epsilon/S$, thus they contribute at most $\epsilon$ to the overall estimation error. This yields an
$\epsilon$-accurate estimation of the dynamics using $O(1/\epsilon^3)$ agents.
The main technical challenge is to achieve the same guarantee using only
$O(1/\epsilon^2)$ agents.

Achieving optimal $\epsilon$-dependency within $O(H)$ phases presents several technical challenges.
Our algorithm exploits the fact that states differ in their reachability - the maximum probability of reaching a state over all possible policies.
A key observation is that exploration effort can be scaled with  reachability, allowing higher errors when estimating the next-state probabilities.
To leverage this insight, as detailed in \Cref{sec:analysis}, we construct a model with an additive error relative to the true model, i.e., $\|q_h(\cdot \mid \pi, \hP) - q_h(\cdot \mid \pi ,P)\|_1 \leq \alpha h$, where $\widehat{P}$ is the estimation and $\alpha$ scales with $\epsilon$.

This structure gives rise to a second, more subtle challenge.
Even when agents are dispatched with policies specifically designed to explore low-reachability states, only a few may actually reach them.
For example, when a few agents arrive at states with reachability of $\alpha/S$, and when there are $\sim S$ such states at timestep $h$, the estimation error at timestep $h+1$ can accumulate an additional factor of $\alpha$. This causes the error to double at each step, leading to exponential growth in the total error.

A third challenge is common to all reward-free algorithms. In each phase, the true dynamics are unknown, yet we must compute policies that explore well. Allowing estimation accuracy to scale with reachability does not affect the value estimates, but it directly affects how the next exploring policies are computed.

% To address the drift from low-reachability states, we introduce an absorbing state $\sinkState$, which prevents estimation drift when few agents arrive at low-reachability states.

Since we only have access to estimated dynamics, we must compute policies that explore effectively despite this uncertainty. To handle both the estimation error and the drift simultaneously, we design the algorithm to target an intermediate, auxiliary dynamics as its goal. Algorithm decisions are evaluated against this auxiliary dynamics, which serves as a reliable proxy for the true dynamics while remaining robust to drift from low-reachability states.

Since the occupancy measures under this auxiliary dynamics differ from the true dynamics by only an additive error at each step $h$, policies that explore effectively with respect to it will also explore effectively in the true dynamics.
This additive structure allows us to leverage a key observation underlying $\algacronyms$: it suffices to control the \emph{policy-weighted error} for any policy $\pi$ — namely, $\|q_h(\cdot \mid \pi, \hat{P})(P_h^\pi - \hat{P}_h^\pi)\|_1$ — to obtain $\epsilon$-accurate estimation of the value.

% Intuitively, this relaxation permits larger errors at states $\pi$ reaches with low probability, since such states contribute little to the occupancy measure.
% This is what enables non-uniform exploration across states, reduces the number of agents per phase, and drives the $\epsilon^{-2}$ dependence in our number of agents.

% 

% Our algorithm addresses this challenge by exploiting the fact that states differ
% in their reachability - the maximum probability of reaching a state, over all possible policies. 
% Highly reachable states are naturally explored by many agents, while low-reachability states are visited by only a few agents.
% Our approach leverages this observation by allowing the estimation error at each state to depend on its reachability probability.

% Intuitively, this relaxation allows larger approximation errors at states that the policy~$\pi$ reaches with low probability, since such states contribute
% little to the induced value function. This observation enables a non-uniform
% exploration across states, and is the key observation that enables us to
% reduce the number of agents in each phase.

More concretely, Algorithm~\algacronyms{} (\cref{alg:our-alg}) gets as input a reachability parameter $\beta > 0$ which defines in each learning phase the group of reachable states.
Then, the algorithm operates in $H$ multi-agent learning phases. In
each phase $i \in [H-1]$, the algorithm explores the dynamics at timestep $i$ (to approximate the transition probability $\realP_i(\cdot| s_i, a_i)$ to $s_{i+1}$).
Meaning, at phase $i$, the algorithm updates the approximation only for transitions at timestep $i$, yields an approximation denoted $\hP_i$ using the previously-learned dynamics estimation $\{\hP_j\}_{j=0}^{i-1}$ (which, when clear from the context we denote as $\hP$). 
%Hence, we denote with $\hP_h$ the estimated transitions from timestep $h$ to $h+1$.
%Hence, at the beginning of phase $i$, the dynamics estimation in use is $\{\hP_j\}_{j=0}^{i-1}$ .

To collect samples to approximate $\realP_i$, the algorithm assigns each agent a random state-action pair to explore. For each $(s,a)$, the algorithm computes a policy
$\hat{\pi}^{i,s,a}$ that maximizes the probability of reaching state~$s$ at timestep
$i$ and taking action~$a$, under the current estimated dynamics~$ \{ \hP_k \}_{k=0}^{i-1}$.

That
is, let 
$\hpi^{i,s} \in \argmax_{\pi} q_i(s \mid \pi,  \{ \hP_k \}_{k=0}^{i-1}) $ and let $\hpi^{i,s,a}(s')$ be $ \hpi^{i,s}(s') $ for $s'\neq s$ and $\hpi^{i,s,a}(s')=a$ for $s'=s$.
This policy can be computed efficiently using standard planning algorithms.
The agents that were assigned $(s,a)$ then execute $\hpi^{i,s,a}$ and collect
trajectories, which are used to construct empirical estimates of the transition
probabilities from timestep~$i$ to timestep~$i+1$, as described next.

To guarantee the quality of the learned dynamics, we restrict exploration to
states that are sufficiently reachable. More formally, we define the set
of \emph{active states} at timestep~$i$ as the states where their maximum reachability probability under $\hP$ is at least $\beta$:
$
    \tS_i
    = \bigl\{ s \in \mathcal S \;\big|\;
    q_i(s \mid \widehat{\pi}^{i,s}, \hP) \ge \beta \bigr\}.
$

We now describe how the empirical dynamics $\widehat P$ are constructed. We
introduce an additional absorbing state, denoted by $\sinkState$, which captures
transitions to states outside the active set of timestep $i$.
% Intuitively, the algorithm only aims to accurately model transitions at timestep $i$ originating from active states, and all other states are treated as negligible at this timestep (a fact we formally justify in the analysis).

Formally, for transitions at timestep~$i$, we define $\hP_i(s' \mid s,a) =N_i(s,a,s')/N_i(s,a)$ if $s \in \tS_i$ and $\one[s' = \sinkState]$ otherwise,
where $N_i(s,a,s')$ denotes the number of agents that in phase $i$ reach state~$s$ at timestep~$i$, take action~$a$, and transition to state~$s'$. The
quantity $N_i(s,a) = \sum_{s'} N_i(s,a,s')$ is the total number of agents that
reach state~$s$ at timestep~$i$ and choose action~$a$.
In each phase $i$, we use the previously computed dynamics $\{\hP\}_{k=0}^{i-1}$ to compute the set $\tS_i$ and the played policies $\widehat{\pi}^{i,s}$ as described earlier. See \Cref{alg:our-alg} for more details.

\begin{algorithm}[t]
% \caption{lala} 
\caption{\algname{} (\algacronyms{})}
\label{alg:our-alg}
\begin{algorithmic}[1]
\State \textbf{Input:} $m$ agents; $\rho = H$ phases; $\beta>0$. 
% \State Define $\beta := ...\epsilon$.
% \State Let $ q_0(s_0,a|\pi, \emptyset):= \mathbbm{1}[\pi(s_0) = a]\;\; \forall a \in \calA$ .
\For{phase $i = 0, \dots, H-1$}
    % \State  $\widehat{\pi}^{i,s,a} \in \argmax_{\pi} q_i(s,a|\pi, \{ \hP_k \}_{k=0}^{i-1})$.
    \State  $\widehat{\pi}^{i,s} \in \argmax_{\pi} q_i(s|\pi, \{ \hP_k \}_{k=0}^{i-1})$.
    \State %Calculate the active states:
    $\tS_i = \{s \in \calS \mid q_i(s \mid \widehat{\pi}^{i,s}, \{\hP_k\}_{k=0}^{i-1}) \geq \beta \}$.
    % , where \\\qquad $\widehat{\pi}^{i,s} \in \argmax_{\pi} q_i(s|\pi, \{ \hP_k \}_{k=0}^{i-1})$.
    % \State Partition the $m$ agents to $\abs{\tS_i} A$ equal size groups, $G^{i,s,a}$.
    % Which maximizes the probability of reaching $(s,a)$ at timestep $i$. 
    \For{each agent $j\in[m]$}
        \State Select state and action with uniform distribution $(s,a)\in \tS_i \times \calA $.
        \State Run policy $\widehat{\pi}^{i,s}$ and choose action $a$ on timestep $i$ and state $s$.
        \State Observe the trajectory $Obs_j$.
    \EndFor
    % \State Define $\hP_i$ according to \Cref{alg:esimate-p}.
    \State  $\hP_i$: empirical dynamics for $\tS_i$ using $\{Obs_j\}_{j\in[m]}$;
     \\
     For $s\notin\tS_i$, transition deterministically  to $\sinkState$.
\EndFor
\State Output $\hP = \{\hP_i\}_{i=0}^{H-1}$.
\end{algorithmic}
\end{algorithm}
Below we give an informal statement of our main result (see \Cref{thm:-we-can-learn-in-h-phases} for formal statement).
\begin{theorem}\label{thm:main-result}
Let any $\epsilon,\delta \in (0,1)$. Run \algacronyms{}~(\Cref{alg:our-alg}) with $\beta = \epsilon/(2 H^2S)$ and \[m = \agentsNumberEpsilonMainText\] agents.
Then, with probability at least $1-\delta$, for any reward function $r$ it holds that
$V^\star_{\realP,r} - V^{\hpi_r}_{\realP,r} \le \epsilon$,
where $\hpi_r \in \argmax_{\pi} V^{\pi}_{\hP,r} $.
\end{theorem}

\paragraph{Discussion.}
Our algorithm is computationally efficient and operates in exactly $H$ learning phases and therefore
achieves an optimal (see \Cref{sec:lower-bound} for the lower bound) parallel time, by design. As \Cref{thm:main-result} states, we obtain the result with a polynomial number of agents, as desired.
However, while the $\epsilon$ dependency is optimal \cite{DBLP:conf/icml/JinKSY20}, the polynomial dependence on $S$ and $H$ is higher than that
of state-of-the-art single-agent reward-free algorithms, such as
\cite{DBLP:conf/icml/MenardDJKLV21}, which achieve a sample complexity of $S^2 H^3 A /\epsilon^2$. 
We note that it remains an open question whether such higher dependencies are an inherent cost of parallelization, or merely an artifact of our analysis. As the focus of our paper is the change of the learning regimes, we leave tightening these dependencies for future work.

\subsection{Analysis}\label{sec:analysis}
In this section, we present the main steps of the upper bound analysis. To achieve high parallelism, we avoid reducing to regret minimization techniques, as these are inherently designed for sequential execution.

The proof reduces to bounding the value difference $\abs{V^\pi_{\realP,r} - V^\pi_{\hP,r}} \leq \epsilon/2$ for every policy $\pi$ and reward $r$ (see \Cref{lemma:dynamics-val-diff-all-policy}), which in turn reduces to controlling the gap between occupancy measures induced by $\hP$ and the true dynamics.

To this end, we introduce an intermediate dynamics $\tP$, which follows $\realP$ on the sets $\tS_h = \{s \mid q_h(s \mid \widehat{\pi}^{h,s}, \hP) \geq \beta\}$ of $\beta$-reachable states under $\hP$, and otherwise transits deterministically to $\sinkState$.

\paragraph{Key step: policy-weighted error bound.}
Our goal is to prove by induction on $h$, that for any policy $\pi$
\begin{equation}\label{eq:alpha-diff}
            \normone{q_h(\cdot \mid \pi, \{\hP_j\}_{j=0}^{h-1}) - q_h(\cdot \mid \pi, \{\tP_j\}_{j=0}^{h-1})} \leq \alpha h,
\end{equation}

where $\alpha = \beta/(3H)$.
The core of the proof is the following bound, which controls the policy-weighted transition error between $\hP$ and $\tP$.
We show the key part of the induction step: 
assuming the induction hypothesis holds for $h$ (\Cref{eq:alpha-diff}), then for every policy $\pi$,
\begin{equation}\label{eq:occ-step}
    \normone{q_h(\cdot \mid \pi, \{\hP_j\}_{j=0}^{h-1})(\hP^\pi_h-\tPpolicy{\pi}_{h})} \leq \alpha.
\end{equation}
% It will derive the induction step and prove 

Since $\tP$ agrees with $\realP$ on $\tS$, we have $q_h(s \mid \hpi^{h,s}, \realP) \ge q_h(s \mid \hpi^{h,s}, \tP)$, so the expected number of agents reaching $s$ at step $h$ is at least $q_h(s \mid \hpi^{h,s}, \tP)\, m/(SA)$. An adapted inequality then yields, with high probability, for all $h,s,a$:
\begin{equation}\label{eq:good-bound}
    \normone{\hP_h(\cdot \mid s, a) - \tP_h(\cdot\mid s, a)}
    \le
    \sqrt{\frac{\log(1/\delta)+S}{ \frac{m}{SA}\, q_h(s \mid \widehat{\pi}^{h,s}, \tP)}}.
\end{equation}
% We henceforth condition on this event.
% We use \Cref{eq:good-bound} to show \Cref{eq:occ-step} holds.

States outside $\tS_h$ contribute nothing to \Cref{eq:occ-step}, since both dynamics send them to $\sinkState$.
Fix $s \in \tS_h$. Using $\hP$ alone, we can compute $\hpi^{h,s}$, the policy that maximizes the visitation probability of $s$ under $\hP$. By the induction hypothesis, $\hpi^{h,s}$ also nearly maximizes the visitation probability under $\tP$: for any policy $\pi$, $q_h(s \mid \pi, \hP) \leq q_h(s \mid \hpi^{h,s}, \hP) \leq q_h(s \mid \hpi^{h,s}, \tP) + \alpha h$.
From \Cref{eq:good-bound} (see \Cref{lem: unaware good estimations})) it suffices to bound
\[
    \bigl(q_h(s \mid \hpi^{h,s}, \tP) + \alpha h\bigr) \cdot
    \sqrt{\frac{\log(1/\delta)+S}{ \tfrac{m}{SA}\, q_h(s \mid \hpi^{h,s}, \tP)}} \le \frac{\alpha}{S}.
\]
% The first term reduces to $\sqrt{q_h(s \mid \hpi^{h,s}, \tP)\,(\log(1/\delta)+S)\,SA/m}$, which is at most $\alpha/(2S)$ by our choice of $m$. For the second term, since we trimmed states with low reachability, it holds that $q_h(s \mid \hpi^{h,s}, \tP) \geq 2\beta/3$.
% Hence, $1/\sqrt{q_h(s \mid \hpi^{h,s}, \tP)} = O(1/\sqrt{\alpha H})$, and $m$ is large enough to yield $\alpha/(2S)$. Combining the two bounds gives \eqref{eq:occ-step}. The remainder of the induction step is deferred to the appendix (\Cref{thm: beta estimated dynamics and estimated dynamics value functions are close}).

The first term reduces to
% \[\sqrt{\frac{q_h(s \mid \hpi^{h,s}, \tP)\,(\log(1/\delta)+S)\,SA}{m}},\]
$\sqrt{(q_h(s \mid \hpi^{h,s}, \tP)\,(\log(1/\delta)+S)\,SA)/m},$
which is at most $\alpha/(2S)$ by our choice of $m$. For the second term, since we trimmed states with low reachability, it holds that $q_h(s \mid \hpi^{h,s}, \tP) \geq 2\beta/3$.
Hence, $1/\sqrt{q_h(s \mid \hpi^{h,s}, \tP)} = O(1/\sqrt{\alpha H})$, and $m$ is large enough to yield $\alpha/(2S)$. Combining the two bounds gives \eqref{eq:occ-step}. The remainder of the induction step is deferred to the appendix (\Cref{thm: beta estimated dynamics and estimated dynamics value functions are close}).

\textbf{Closing the gap: $\tP$ approximates $\realP$.}
Note that $\tP$ is defined with respect to the \emph{reachability under the estimated dynamics} rather than the reachability under the true dynamics. To bridge this gap, we introduce an auxiliary dynamics $\bP$ and use it to compare $\tP$ with $\realP$. The auxiliary dynamics is defined inductively on the \emph{true} inductive-low-reachability set $\bS_h := \{s \mid \max_\pi q_h(s \mid \pi, \{\bP_j\}_{j=0}^{h-1}) \geq 2\beta\}$ (where $\bP_0 = P_0$).

Standard arguments show $\bP$ is close to $\realP$.
Therefore it remains to show that $\bP$ is close to $P^1$.
For that it suffices to prove $\bS_h \subseteq \tS_h$ for all $h$.

By induction: Assume $s \in \bS_{h+1}$. 
Let $\pi$ be a witness for $s \in \bS_{h+1}$, so $2\beta \leq q_{h+1}(s\mid\pi,\bP)$. Applying the induction hypothesis, we can show $q_{h+1}(s\mid\pi,\bP)\leq q_{h+1}(s\mid\pi,\tP)$ (\Cref{lem: occupancy measure is higher in beta estimated dynamics than beta dynamics}).
Here the bound we derived in \Cref{eq:alpha-diff} plays another role.
By \Cref{eq:alpha-diff}, the occupancy measures under $\tP$ and $\hP$ differ by at most $\alpha (h+1) \leq \beta/3$ for any policy, and in particular, this bound holds for the policy $\pi$ that certifies $s \in \bS_{h+1}$.
Therefore, 
\[2\beta \leq q_{h+1}(s\mid\pi, \bP) \leq q_{h+1}(s\mid\pi, \tP) \leq q_{h+1}(s\mid \pi, \hP) + \beta/3 \leq q_{h+1}(s\mid\widehat{\pi}^{h,s}, \hP) + \beta/3,\]
where the last inequality is since $\hpi^{h,s}$ maximizes the probability of reaching state $s$ at timestep $h$. Rearranging the terms, we obtain that $s \in \tS_{h+1}$.
For more details see \Cref{lem: beta estimated dynamics states contains beta dynamics states}.

Together, these steps establish that $\hP$ approximates $\realP$ over all policies and rewards, completing the proof sketch for the upper bound.

\begin{remark}
We believe that our exploration strategy may extend to linear MDPs. Conceptually, one would replace states with feature basis vectors, and the set-visitation measure \cite{DBLP:conf/icml/WagenmakerCSDJ22a-reward-free-linear-mdp} is analogous to our notion of reachability. However, existing reward-free algorithms for linear MDPs are inherently sequential, as most rely on bonuses or on regret-minimization techniques with synthetic rewards for exploration. Since our framework does not employ standard regret minimization techniques, designing a parallelized counterpart introduces additional challenges.
% ($\omega_h^\pi(\mathcal{X}) := \mathbb{P}_\pi[\boldsymbol{\phi}(s_h, a_h) \in \mathcal{X}]$ for $\mathcal{X}\subseteq \mathbb{R}^d$ where $\boldsymbol{\phi}$ is the feature embedding )
% Nevertheless, we believe that under an appropriate formulation, an analogous result can be obtained. Since the primary focus of this work is the trade-off between wall-clock time and the number of agents, we leave the formal development of this extension as a future work.
\end{remark}

\section{Lower Bound }
\label{sec:lower-bound}
In this section, we present a lower bound proving that when the number of learning phases is too small (i.e., significantly smaller than $H$), achieving reward-free guarantees necessarily requires an exponential number of agents.
This result complements our upper bound, which demonstrates that efficient reward-free learning using only a polynomial number of agents is achievable with $H$ learning phases.
Taken together, the upper and lower bounds reveal a regime change in learning efficiency at the horizon scale $H$, separating regimes of polynomial and exponential agent complexity.

Our lower bound is stated next (see \cref{apndx:sec:lower-bound} for full proof).
\begin{theorem}
\label{thm:main-lower-bound-many-phases}
    Let $\texttt{A}$ be a cooperative multi-agent algorithm that runs for $\rho$ learning phases with $m$ agents in each phase.
    Assume it satisfies the reward-free guarantee for $\epsilon = 0.1$ and $\delta = 0.55$, i.e., it outputs dynamics $\hP$ such that for any reward function $r$,
    $\prob{V^\star_{P,r} - V^{\hpi_r}_{P,r} < 0.1} > 0.45$,
    where $\hpi_r \in \argmax_\pi \{ V^\pi_{\hP,r} \}$.
    Then, $\texttt{A}$ must deploy\\ 
    \[m = \Omega\left(\frac{A^{(H-1)/\rho}}{\rho}\right)\]
    agents in each learning phase.
\end{theorem}

\begin{figure}[t]
\centering
\begin{tikzpicture}[
    x=2.3cm, y=1.7cm,
    state/.style={circle, draw, minimum size=9mm, inner sep=0pt},
    lab/.style={midway, fill=white, inner sep=1pt}
]

% ---- nodes: two rows (top: starred, bottom: sink), one column per time h ----
\node[state] (s0)  at (0,1) {$\starState$};
\node[state] (k0)  at (0,0) {$\sinkState$};

\node[state] (s1)  at (1,1) {$\starState$};
\node[state] (k1)  at (1,0) {$\sinkState$};

\node[state] (s2)  at (2,1) {$\starState$};
\node[state] (k2)  at (2,0) {$\sinkState$};

\node[state] (s3)  at (3,1) {$\starState$};
\node[state] (k3)  at (3,0) {$\sinkState$};

% ---- time labels ----
\node[above=2mm of s0] {$h=0$};
\node[above=2mm of s1] {$h=1$};
\node[above=2mm of s2] {$h=2$};
\node[above=2mm of s3] {$h=3$};

% ---- transitions ----
% correct action keeps you in s^*
\draw[->] (s0) -- node[lab, above] {$a_0^\star$} (s1);
\draw[->] (s1) -- node[lab, above] {$a_1^\star$} (s2);
\draw[->] (s2) -- node[lab, above] {$a_2^\star$} (s3);

% any other action sends you to sink
\draw[->] (s0) -- node[lab, sloped, above] {$a\neq a_0^\star$} (k1);
\draw[->] (s1) -- node[lab, sloped, above] {$a\neq a_1^\star$} (k2);
\draw[->] (s2) -- node[lab, sloped, above] {$a\neq a_2^\star$} (k3);

% sink is absorbing (any action)
\draw[->] (k1) -- node[lab, below] {$\calA$} (k2);
\draw[->] (k2) -- node[lab, below] {$\calA$} (k3);

\end{tikzpicture}
\caption{\label{fig:key-dynamics}
Hidden key for the lower bound: only $a_h^\star$ stays in $\starState$; otherwise transition to $\sinkState$.}
\end{figure}

To build intuition, we sketch the proof for $\rho = 1$, and defer the proof for general $\rho$ to \Cref{thm:lower-bound-many-phases}.

\begin{proof}[Proof sketch for $\rho = 1$]
Our lower-bound instance is an MDP whose dynamics encode a hidden key of length $H$, with symbols drawn from the action set $\calA$.
Intuitively, an algorithm deploying $m$ agents can uncover at most an additional $\log_A(m)$ actions of this key in each learning phase.
Hence, to achieve the desired accuracy for a reward function $r$ that assigns a reward of $1$ only upon reaching the end of the key and $0$ otherwise, the algorithm must therefore uncover the entire key.

More formally, consider the state space $\calS=\{\sinkState, s^\star\}$, and the initial state is $s_0 = s^\star$.
A random key, denoted $a^\star_0, \ldots, a^\star_{H-1}$, is chosen and fixed. The dynamics are defined using this key as follows. For all $h \in [H-1]$, at state $s^\star$, taking action $a^\star_h$ deterministically transitions back to $s^\star$, while taking any other action transitions to $\sinkState$.
From state $\sinkState$, taking any action deterministically transitions back to $\sinkState$, meaning that once an agent reaches $\sinkState$, it cannot escape.
See \Cref{fig:key-dynamics} for a diagram of the dynamics for $H=3$.

Consequently, when a small number of agents are deployed, recovering the full key, and thus the dynamics, is impossible.
To see this, for all $h \in [H]$, we denote by $\goodAgents_h \subseteq [m]$ the subset of the agents that reached state $\starState$ at timestep $h$.
We first show that for all fixed $h$, $\Expect{\abs{\goodAgents_{{h}+1}}} = \Expect{\abs{\goodAgents_{{h}}}}/A$.

Let $a^i_h$ denote the action that agent $i$ played at timestep $h$. When conditioning on  $\goodAgents_{h}$ we obtain
    \begingroup\allowdisplaybreaks
    \begin{align*}
    &\Expect{\abs{\goodAgents_{h+1}}\mid \goodAgents_{h}} =\sum_{a\in\calA}\sum_{i\in\goodAgents_{h}}\prob{a^i_h = a , a = a^\star_h \mid \goodAgents_{h}}
    %\\& 
    %=\sum_{a\in\calA}\sum_{i\in\goodAgents_{h}}\prob{a^i_h = a \mid \goodAgents_{h}, a = a^\star_h }\prob{a = a^\star_h }
    \\&=\frac{1}{A}\sum_{i\in\goodAgents_{h}}\sum_{a\in\calA}\prob{a^i_h = a \mid \goodAgents_{h},a = a^\star_h} 
    %=\frac{1}{A}\sum_{i\in\goodAgents_{h}}1
    = \frac{\abs{\goodAgents_{h}}}{A},
    \end{align*}
    \endgroup
where we use the fact that $a^i_h$ and $a^\star_h$ are independent, and that each agent chooses only one action. Thus 
$\sum_{a\in\calA} \prob{a^i_h = a \mid \goodAgents_{h}, a = a^\star_h}
=  1$. 
Next, we take the expectation over $\goodAgents_{h}$ and combine with the above to obtain  
% $
%     \Expect{\abs{\goodAgents_{h+1}}} 
%     =
%     \Expect{\Expect{\abs{\goodAgents_{h+1}}\mid \goodAgents_{h}}}
%     =
%     \Expect{\abs{\goodAgents_{h}}}/A
%     =m/A^{h+1},
% $
$$
    \Expect{\abs{\goodAgents_{h+1}}} 
    =
    \Expect{\Expect{\abs{\goodAgents_{h+1}}\mid \goodAgents_{h}}}
    =
    \frac{\Expect{\abs{\goodAgents_{h}}}}{A}
    =\frac{m}{A^{h+1}},
$$
where the last equality can be shown by induction.
Thus, when $m \ll A^{H-1}$, there is a constant probability that $\goodAgents_H$ is empty; therefore, the algorithm fails to recover the full key.
Consequently, for our reward function $r$ (given by $r_h(s,a) = \one(h = H-1,\; s = s^\star,\; a = a^\star_{H-1})$), the optimal policy achieves a value of $1$, while the algorithm has a constant probability of failure, yielding a value strictly smaller than $0.9$. This contradicts the assumed $0.1$-optimality of $\texttt{A}$, proving the theorem.
\end{proof}

\begin{remark}
The horizon dependency we establish reflects a regime change, though the transition is gradual rather than sharp.
One can leverage \algacronyms{} to learn in $H/2$ phases.
% Even when $\rho = H/2$, a similar upper bound can be derived.
By condensing an MDP with $H$ layers and $A$ actions into one with $H/2$ layers and $A^2$ actions, the reformulated MDP can be learned in $H/2$ phases (potentially requiring $A$ times more agents). This recovers all even layers within $H/2$ phases. Since every odd layer can be learned in parallel with its subsequent even layer—at the cost of doubling the number of agents—the entire MDP can be learned in $H/2$ phases.
As our lower bound shows, setting $\rho = H/2$ yields an agent complexity of $\sim A^2/H$. The regime where our lower bound is most significant is where  $\rho \ll H$.
\end{remark}

\section{Discussion and Future Work}

In this work, we introduce the problem of cooperative multi-agent reward-free exploration, in which a centralized learning algorithm deploys multiple agents in parallel to learn unknown environment dynamics.
We identify a fundamental trade-off between parallel time and agent complexity that naturally arises in this setting and analyze it.
Our results present a regime change when the number of learning phases scales with $H$, the horizon of the learned MDP.
For $H$ learning phases we show a computationally efficient algorithm that requires
$\tilde{O}(S^6 H^6 A / \epsilon^2)$ agents.
In contrast, when the number of learning phases is reduced to $\rho < H$, we prove that any algorithm must deploy at least
$\Omega(A^{(H-1)/\rho} / \rho)$ agents.
Together, these results demonstrate two distinct learning regimes: one with polynomial and one with exponential agent complexity, determined by the number of learning phases.

Several directions for future work remain. First, reducing the polynomial dependence on $S$ and $H$ in our upper bound — both under general conditions and under additional assumptions — would lead to a tighter characterization of the polynomial regime. Second, extending our results to settings with large or continuous state spaces is another important direction. We hope this work inspires deeper understanding of cooperative multi-agent exploration and further investigation into the fundamental limits of parallel reinforcement learning.

\clearpage
\beforeappendix

% \section*{Acknowledgments}
% This project is supported by the European Research Council (ERC) under the European Union’s Horizon 2020 research and innovation program (grant agreement No. 882396), by the Israel Science Foundation and the Yandex Initiative for Machine Learning at Tel Aviv University and by a grant from the Tel Aviv University Center for AI and Data Science (TAD).
% OL is also supported by the Google PhD fellowship award (2025).
%%%%%%%

%%%% bibliography %%%%
\clearpage
\bibliographystyle{abbrvnat}
\bibliography{coop_reward_free}

@InProceedings{lancewicki2022-coop-mdp,
  title = 	 {Cooperative Online Learning in Stochastic and Adversarial {MDP}s},
  author =       {Lancewicki, Tal and Rosenberg, Aviv and Mansour, Yishay},
  booktitle = 	 {Proceedings of the 39th International Conference on Machine Learning},
  year = 	 {2022}
}

@inproceedings{DBLP:conf/icml/JinKSY20,
  author       = {Chi Jin and
                  Akshay Krishnamurthy and
                  Max Simchowitz and
                  Tiancheng Yu},
  title        = {Reward-Free Exploration for Reinforcement Learning},
  booktitle    = {Proceedings of the 37th International Conference on Machine Learning,
                  {ICML}},
  year         = {2020},
      }

@book{MannorMT-RLbook,

  
  author = {Mannor, Shie and Mansour, Yishay and Tamar, Aviv},

  title = {Reinforcement Learning:  Foundations},

  year = {2022},

  publisher = {-}

}

@inproceedings{DBLP:conf/icml/MenardDJKLV21,
  author       = {Pierre M{\'{e}}nard and
                  Omar Darwiche Domingues and
                  Anders Jonsson and
                  Emilie Kaufmann and
                  Edouard Leurent and
                  Michal Valko},
  title        = {Fast active learning for pure exploration in reinforcement learning},
  booktitle    = {Proceedings of the 38th International Conference on Machine Learning,
                  {ICML}},
  year         = {2021},
      }

@inproceedings{DBLP:conf/alt/KaufmannMDJLV21,
  author       = {Emilie Kaufmann and
                  Pierre M{\'{e}}nard and
                  Omar Darwiche Domingues and
                  Anders Jonsson and
                  Edouard Leurent and
                  Michal Valko},
  title        = {Adaptive Reward-Free Exploration},
  booktitle    = {Algorithmic Learning Theory},
  year         = {2021},
      }

@book{puterman2014markov,
  title={Markov decision processes: discrete stochastic dynamic programming},
  author={Puterman, Martin L},
  year={2014},
  publisher={John Wiley \& Sons}
}

@inproceedings{DBLP:conf/icml/QiaoYM022loglogt,
  author       = {Dan Qiao and
                  Ming Yin and
                  Ming Min and
                  Yu{-}Xiang Wang},
  title        = {Sample-Efficient Reinforcement Learning with loglog(T) Switching Cost},
  booktitle    = {International Conference on Machine Learning, {ICML}},
  year         = {2022},
      }

@inproceedings{DBLP:conf/nips/ZhangJZJ22Batch,
  author       = {Zihan Zhang and
                  Yuhang Jiang and
                  Yuan Zhou and
                  Xiangyang Ji},
  title        = {Near-Optimal Regret Bounds for Multi-batch Reinforcement Learning},
  booktitle    = {Advances in Neural Information Processing Systems},
  year         = {2022},
      }

@inproceedings{DBLP:conf/nips/Zhang0J20ModelFreeParallel,
  author       = {Zihan Zhang and
                  Yuan Zhou and
                  Xiangyang Ji},
  title        = {Almost Optimal Model-Free Reinforcement Learning via Reference-Advantage
                  Decomposition},
  booktitle    = {Advances in Neural Information Processing Systems},
  year         = {2020},
      }

@inproceedings{DBLP:conf/iclr/HuangCZQJL22Deployment,
  author       = {Jiawei Huang and
                  Jinglin Chen and
                  Li Zhao and
                  Tao Qin and
                  Nan Jiang and
                  Tie{-}Yan Liu},
  title        = {Towards Deployment-Efficient Reinforcement Learning: Lower Bound and
                  Optimality},
  booktitle    = {The Tenth International Conference on Learning Representations, {ICLR}},
    year         = {2022},
      }

@inproceedings{DBLP:conf/nips/QianHS24CmdpSimchiLevi,
  author       = {Jian Qian and
                  Haichen Hu and
                  David Simchi{-}Levi},
  title        = {Offline Oracle-Efficient Learning for Contextual MDPs via Layerwise
                  Exploration-Exploitation Tradeoff},
  booktitle    = {Advances in Neural Information Processing Systems},
  year         = {2024},
      }

@inproceedings{DBLP:conf/uai/Asghari0N20,
  author       = {Seyed Mohammad Asghari and
                  Yi Ouyang and
                  Ashutosh Nayyar},
  title        = {Regret Bounds for Decentralized Learning in Cooperative Multi-Agent
                  Dynamical Systems},
  booktitle    = {Proceedings of the Thirty-Sixth Conference on Uncertainty in Artificial
                  Intelligence, {UAI}},
  year         = {2020},
      }

@inproceedings{DBLP:conf/l4dc/HsuP25Safe,
  author       = {Hao{-}Lun Hsu and
                  Miroslav Pajic},
  title        = {Safe Cooperative Multi-Agent Reinforcement Learning with Function
                  Approximation},
  booktitle    = {7th Annual Learning for Dynamics {\&} Control Conference},
  year         = {2025},
      }

@inproceedings{DBLP:conf/nips/AgarwalKKS20-reward-free-low-rank,
  author       = {Alekh Agarwal and
                  Sham M. Kakade and
                  Akshay Krishnamurthy and
                  Wen Sun},
  title        = {{FLAMBE:} Structural Complexity and Representation Learning of Low
                  Rank MDPs},
  booktitle    = {Advances in Neural Information Processing Systems},
  year         = {2020},
      }

@inproceedings{DBLP:conf/nips/WangDYS20-reward-free-linear-function-approx,
  author       = {Ruosong Wang and
                  Simon S. Du and
                  Lin F. Yang and
                  Ruslan Salakhutdinov},
  title        = {On Reward-Free Reinforcement Learning with Linear Function Approximation},
  booktitle    = {Advances in Neural Information Processing Systems},
  year         = {2020},
      }

@inproceedings{DBLP:conf/icml/ZhangSUWAS22-reward-free-block-mdp,
  author       = {Xuezhou Zhang and
                  Yuda Song and
                  Masatoshi Uehara and
                  Mengdi Wang and
                  Alekh Agarwal and
                  Wen Sun},
  title        = {Efficient Reinforcement Learning in Block MDPs: {A} Model-free Representation
                  Learning approach},
  booktitle    = {International Conference on Machine Learning, {ICML}},
  year         = {2022},
      }

@inproceedings{DBLP:conf/nips/Chen0K0A22-reward-free-function-approx,
  author       = {Jinglin Chen and
                  Aditya Modi and
                  Akshay Krishnamurthy and
                  Nan Jiang and
                  Alekh Agarwal},
  title        = {On the Statistical Efficiency of Reward-Free Exploration in Non-Linear
                  {RL}},
  booktitle    = {Advances in Neural Information Processing Systems},
  year         = {2022},
      }

@inproceedings{DBLP:conf/icml/MiryoosefiJ22-reward-free-constraints,
  author       = {Sobhan Miryoosefi and
                  Chi Jin},
  title        = {A Simple Reward-free Approach to Constrained Reinforcement Learning},
  booktitle    = {International Conference on Machine Learning, {ICML}},
  year         = {2022},
      }

@inproceedings{DBLP:conf/iclr/HuCH23-reward-free-linear-mdp,
  author       = {Pihe Hu and
                  Yu Chen and
                  Longbo Huang},
  title        = {Towards Minimax Optimal Reward-free Reinforcement Learning in Linear
                  MDPs},
  booktitle    = {The Eleventh International Conference on Learning Representations,
                  {ICLR}},
  year         = {2023},
      }

@inproceedings{DBLP:conf/iclr/ChengHL023-reward-free-low-rank,
  author       = {Yuan Cheng and
                  Ruiquan Huang and
                  Yingbin Liang and
                  Jing Yang},
  title        = {Improved Sample Complexity for Reward-free Reinforcement Learning
                  under Low-rank MDPs},
  booktitle    = {The Eleventh International Conference on Learning Representations,
                  {ICLR}},
  year         = {2023},
      }

@inproceedings{DBLP:conf/nips/MhammediBFR23-reward-free-low-rank,
  author       = {Zakaria Mhammedi and
                  Adam Block and
                  Dylan J. Foster and
                  Alexander Rakhlin},
  title        = {Efficient Model-Free Exploration in Low-Rank MDPs},
  booktitle    = {Advances in Neural Information Processing Systems},
  year         = {2023},
      }

@inproceedings{DBLP:conf/icml/AmortilaFK24-reward-free-block-low-rank-and-l1-coverage,
  author       = {Philip Amortila and
                  Dylan J. Foster and
                  Akshay Krishnamurthy},
  title        = {Scalable Online Exploration via Coverability},
  booktitle    = {Forty-first International Conference on Machine Learning, {ICML}},
  year         = {2024},
      }

@inproceedings{DBLP:conf/icml/WagenmakerCSDJ22a-reward-free-linear-mdp,
  author       = {Andrew J. Wagenmaker and
                  Yifang Chen and
                  Max Simchowitz and
                  Simon S. Du and
                  Kevin G. Jamieson},
  title        = {Reward-Free {RL} is No Harder Than Reward-Aware {RL} in Linear Markov
                  Decision Processes},
  booktitle    = {International Conference on Machine Learning, {ICML}},
  year         = {2022},
      }

@inproceedings{DBLP:conf/nips/ZhangZG21-reward-free-linear-function-approx,
  author       = {Weitong Zhang and
                  Dongruo Zhou and
                  Quanquan Gu},
  title        = {Reward-Free Model-Based Reinforcement Learning with Linear Function
                  Approximation},
  booktitle    = {Advances in Neural Information Processing Systems},
  year         = {2021},
      }

@inproceedings{
zhang2025reward-free-deployment-linear-func-approx,
title={Deployment Efficient Reward-Free Exploration with Linear Function Approximation},
author={Zihan Zhang and Yuxin Chen and Jason D. Lee and Simon Shaolei Du and Lin Yang and Ruosong Wang},
booktitle={The Thirty-ninth Annual Conference on Neural Information Processing Systems},
year={2025},
}

@InProceedings{pmlr-v139-zhang21e-reward-agnostic,
  title = 	 {Near Optimal Reward-Free Reinforcement Learning},
  author =       {Zhang, Zihan and Du, Simon and Ji, Xiangyang},
  booktitle = 	 {Proceedings of the 38th International Conference on Machine Learning},
  year = 	 {2021},
  }

@inproceedings{DBLP:conf/nips/ZhaoHG24,
  author       = {Heyang Zhao and
                  Jiafan He and
                  Quanquan Gu},
  title        = {A Nearly Optimal and Low-Switching Algorithm for Reinforcement Learning
                  with General Function Approximation},
  booktitle    = {Advances in Neural Information Processing Systems},
  year         = {2024},
      }

@Inbook{Zhang2021MARL,
author="Zhang, Kaiqing
and Yang, Zhuoran
and Ba{\c{s}}ar, Tamer",
title="Multi-Agent Reinforcement Learning: A Selective Overview of Theories and Algorithms",
bookTitle="Handbook of Reinforcement Learning and Control",
year="2021",
publisher="Springer International Publishing",
pages="321--384",
}

@inproceedings{DBLP:conf/emnlp/0005ZC24,
  author       = {Zheng Zhao and
                  Yftah Ziser and
                  Shay B. Cohen},
  title        = {Layer by Layer: Uncovering Where Multi-Task Learning Happens in Instruction-Tuned
                  Large Language Models},
  booktitle    = {Proceedings of the 2024 Conference on Empirical Methods in Natural
                  Language Processing, {EMNLP}},
  year         = {2024},
    }

@article{DBLP:journals/jmlr/ChungHLZTFL00BW24,
  author       = {Hyung Won Chung and
                  Le Hou and
                  Shayne Longpre and
                  Barret Zoph and
                  Yi Tay and
                  William Fedus and
                  Yunxuan Li and
                  Xuezhi Wang and
                  Mostafa Dehghani and
                  Siddhartha Brahma and
                  Albert Webson and
                  Shixiang Shane Gu and
                  Zhuyun Dai and
                  Mirac Suzgun and
                  Xinyun Chen and
                  Aakanksha Chowdhery and
                  Alex Castro{-}Ros and
                  Marie Pellat and
                  Kevin Robinson and
                  Dasha Valter and
                  Sharan Narang and
                  Gaurav Mishra and
                  Adams Yu and
                  Vincent Y. Zhao and
                  Yanping Huang and
                  Andrew M. Dai and
                  Hongkun Yu and
                  Slav Petrov and
                  Ed H. Chi and
                  Jeff Dean and
                  Jacob Devlin and
                  Adam Roberts and
                  Denny Zhou and
                  Quoc V. Le and
                  Jason Wei},
  title        = {Scaling Instruction-Finetuned Language Models},
  journal      = {J. Mach. Learn. Res.},
    year         = {2024}}

@Article{robotics-survey-2025-orr,
AUTHOR = {Orr, James and Dutta, Ayan},
TITLE = {Multi-Agent Deep Reinforcement Learning for Multi-Robot Applications: A Survey},
JOURNAL = {Sensors},
VOLUME = {23},
YEAR = {2023},
}

@article{DBLP:journals/tmlr/KaufmannWBH25-RLHF,
  author       = {Timo Kaufmann and
                  Paul Weng and
                  Viktor Bengs and
                  Eyke H{\"{u}}llermeier},
  title        = {A Survey of Reinforcement Learning from Human Feedback},
  journal      = {Trans. Mach. Learn. Res.},
  year         = {2025},
}

@article{DBLP:journals/corr/HeessTSLMWTEWER17-dppo,
  author       = {Nicolas Heess and
                  Dhruva TB and
                  Srinivasan Sriram and
                  Jay Lemmon and
                  Josh Merel and
                  Greg Wayne and
                  Yuval Tassa and
                  Tom Erez and
                  Ziyu Wang and
                  S. M. Ali Eslami and
                  Martin A. Riedmiller and
                  David Silver},
  title        = {Emergence of Locomotion Behaviours in Rich Environments},
  journal      = {CoRR},
  year         = {2017},
  eprinttype   = {arXiv},
  eprint       = {1707.02286},
}

@inproceedings{DBLP:conf/icml/DimakopoulouR18-concurrent-rl,
  author       = {Maria Dimakopoulou and
                  Benjamin Van Roy},
  editor       = {Jennifer G. Dy and
                  Andreas Krause},
  title        = {Coordinated Exploration in Concurrent Reinforcement Learning},
  booktitle    = {Proceedings of the 35th International Conference on Machine Learning,
                  {ICML}},
  series       = {Proceedings of Machine Learning Research},
  publisher    = {{PMLR}},
  year         = {2018},
}

@inproceedings{DBLP:conf/nips/DimakopoulouOR18-concurrent-rl-scalable,
  author       = {Maria Dimakopoulou and
                  Ian Osband and
                  Benjamin Van Roy},
  editor       = {Samy Bengio and
                  Hanna M. Wallach and
                  Hugo Larochelle and
                  Kristen Grauman and
                  Nicol{\`{o}} Cesa{-}Bianchi and
                  Roman Garnett},
  title        = {Scalable Coordinated Exploration in Concurrent Reinforcement Learning},
  booktitle    = {Advances in Neural Information Processing Systems 31: Annual Conference
                  on Neural Information Processing Systems 2018, NeurIPS},
  year         = {2018},
}

@inproceedings{DBLP:conf/iclr/EysenbachGIL19-reward-free-first,
  author       = {Benjamin Eysenbach and
                  Abhishek Gupta and
                  Julian Ibarz and
                  Sergey Levine},
  title        = {Diversity is All You Need: Learning Skills without a Reward Function},
  booktitle    = {7th International Conference on Learning Representations, {ICLR}},
  year         = {2019},
}

@article{DBLP:journals/tnn/HaoYTBLMLW24-multi-agent-exploration-survey,
  author       = {Jianye Hao and
                  Tianpei Yang and
                  Hongyao Tang and
                  Chenjia Bai and
                  Jinyi Liu and
                  Zhaopeng Meng and
                  Peng Liu and
                  Zhen Wang},
  title        = {Exploration in Deep Reinforcement Learning: From Single-Agent to Multiagent
                  Domain},
  journal      = {{IEEE} Trans. Neural Networks Learn. Syst.},
  year         = {2024},
}

@inproceedings{DBLP:conf/nips/LiuLTZ18-horizon,
  author       = {Qiang Liu and
                  Lihong Li and
                  Ziyang Tang and
                  Dengyong Zhou},
  editor       = {Samy Bengio and
                  Hanna M. Wallach and
                  Hugo Larochelle and
                  Kristen Grauman and
                  Nicol{\`{o}} Cesa{-}Bianchi and
                  Roman Garnett},
  title        = {Breaking the Curse of Horizon: Infinite-Horizon Off-Policy Estimation},
  booktitle    = {Advances in Neural Information Processing},
  year         = {2018},
}

@inproceedings{DBLP:conf/nips/FosterBM24-horizon,
  author       = {Dylan J. Foster and
                  Adam Block and
                  Dipendra Misra},
  editor       = {Amir Globersons and
                  Lester Mackey and
                  Danielle Belgrave and
                  Angela Fan and
                  Ulrich Paquet and
                  Jakub M. Tomczak and
                  Cheng Zhang},
  title        = {Is Behavior Cloning All You Need? Understanding Horizon in Imitation
                  Learning},
  booktitle    = {Advances in Neural Information Processing},
  year         = {2024},
}

@inproceedings{DBLP:journals/jmlr/RossB10-horizon,
  author       = {St{\'{e}}phane Ross and
                  Drew Bagnell},
  editor       = {Yee Whye Teh and
                  D. Mike Titterington},
  title        = {Efficient Reductions for Imitation Learning},
  booktitle    = {Proceedings of the Thirteenth International Conference on Artificial
                  Intelligence and Statistics, {AISTATS} 2010, Chia Laguna Resort, Sardinia,
                  Italy, May 13-15, 2010},
  year         = {2010},
}

@inproceedings{DBLP:conf/colt/ZhangJD21-horizon,
  author       = {Zihan Zhang and
                  Xiangyang Ji and
                  Simon S. Du},
  editor       = {Mikhail Belkin and
                  Samory Kpotufe},
  title        = {Is Reinforcement Learning More Difficult Than Bandits? {A} Near-optimal
                  Algorithm Escaping the Curse of Horizon},
  booktitle    = {Conference on Learning Theory, {COLT} 2021, 15-19 August 2021, Boulder,
                  Colorado, {USA}},
  year         = {2021},
}

@article{DBLP:journals/ior/KallusU22-horizon,
  author       = {Nathan Kallus and
                  Masatoshi Uehara},
  title        = {Efficiently Breaking the Curse of Horizon in Off-Policy Evaluation
                  with Double Reinforcement Learning},
  journal      = {Oper. Res.},
  year         = {2022},
}

\clearpage
\appendix
\onecolumn

% \section{Proofs for \Cref{sec:analysis}}

\section{Upper bound}
\label{sec:approx-model-appendix}

In this section we prove our upper bound.

% In phase $h\in\{0,\dots,H-1\}$ we estimate the transition function for timestep $h$. To prevent the propagation of errors, we mitigate the divergence caused by states for which we lack a reliable estimation. 

% % To address this, we build the estimated 

% We construct an intermediate dynamics, denoted as $\tP$, which is built inductively. In this model, transitions are restricted to states that are estimated to be $\beta$-reachable from the previous timestep's reliable set.
% It is important to note that from the states that are estimated as states with high reachability, the transition function of $\tP$ is the true transition function, i.e., the same as $\realP$.
% We then provide a good estimation, $\hP$, for this $\tP$, even when our interaction is from the true dynamics $\realP$.

% To establish the validity of this estimation, we utilize a comparison to $\bP$, the true $\morebeta$-reachable inductive dynamics. While $\bP$ is also constructed inductively, it uses the true transition dynamics to define reachability.
% Finally, we demonstrate that $\tP$ serves as a faithful approximation to $\bP$, which is a good approximation of the real $\realP$ as well.

\subsection{Definitions}

\begin{definition}
    We denote the value function of policy $\pi$ under dynamics $P$ and reward $r$, with $V^\pi(s_0;P,r) := V^\pi_{P,r}$.
    The optimal value function is denoted with $V^\star(s_0;P,r)$.
\end{definition}

\begin{definition}
    We denote with $\alpha$ the error of the occupancy measures estimation.
    The theorems will assume that $\alpha$ scales with $\epsilon$ to ensure $\epsilon$ optimality.
\end{definition}

\begin{definition}
    We define the $\beta$ parameter to be,
    \[
    \beta := 3 \alpha H.
    \]
    This parameter will be used to define the states we ignore in our transition function estimation.
    
\end{definition}

\begin{remark}
    \label{rem:appndx:deterministic-reward}
    Without loss of generality, the reward functions considered in the appendix are deterministic.
    Since our analysis concerns the value function, which depends only on the expected reward, we can restrict attention to deterministic reward functions.
\end{remark}

\begin{remark}
\label{rem:appndx:stochastic-policy}
    We generalize the policy definition and allow policies to be stochastic, even though our algorithm \algacronyms{} (\Cref{alg:our-alg}) uses only deterministic policies.
    Hence, we could restrict the setting to deterministic policies only, but we choose to show that the algorithm returns a good estimate of the dynamics for every stochastic policy.
    For example, when we compare the dynamics the algorithm estimated with the true dynamics, we need to consider every stochastic policy, and not just deterministic, see for example the proof of \Cref{thm: beta estimated dynamics and estimated dynamics value functions are close}.
    In other words, we assume $\pi_h(s)$ induces a probability over the actions.
    The value function is over this randomness as well as the randomness of the MDP.
\end{remark}

\subsubsection{The estimated dynamics $\hP$}

We define the estimated dynamics, $\hP$, with the empirical mean. To handle states $s$ where we do not have a good estimation for the transition functions $P_h(\cdot\mid s, \cdot)$, we introduce a virtual sink state $\sinkState$ and a criterion for reachability.
We continue with the empirical probability only when we estimate the state to be reachable with probability at least $\beta$, from the dynamics we built in the previous timestep.

\begin{definition}
Let $\hq(s\mid \pi) := q_h(s \mid \pi, \hP)$ be the \textbf{occupancy measure} of state $s\neq\sinkState$ at timestep $h$, when policy $\pi$ played over dynamics $\hP$.
I.e., $q_h(s\mid\pi, \hP)$ is the probability to reach state $s$ at timestep $h$, under dynamic $\hP$ with policy $\pi$.
We also denote $\hq$ where it is clear what the policy is.
Note that we take $\hq$ over $\calS$ and without $\sinkState$, just for the ease of notation.
\end{definition}

\begin{definition}
    We denote with $\widehat{\pi}^{h,s}$ the policy that under the estimated dynamics $\hP$, reaches state $s$ at step $h$ with maximum probability.
    I.e., $\widehat{\pi}^{h,s} = \argmax_\pi\{q_h(s\mid \pi, \hP)\}$.
    % For brevity, we denote $\hq := q_h(\cdot \mid \pi, \hP)$.  
\end{definition}

\begin{definition}
Let $\tS_h$ denote the set of states that are empirically $\beta$-reachable at timestep $h$:
\begin{equation}
    \label{eq: reachable states in beta estimated dynamics}
    \tS_h := \{s \in \calS \mid q_h(s \mid \widehat{\pi}^{h,s}, \hP) \geq \beta \}.
\end{equation}

\end{definition}

The \textbf{transition function} $\hP_h$ is defined by:

\[
\begin{array}{c c | c}
s & s' & \hP_h(s' \mid s, a) \\
\hline
\rule{0pt}{15pt} \tS_h & \calS & \frac{N_h(s, a, s')}{N_h(s, a)} \\[6pt]
\calS \setminus \tS_h & \sinkState & 1 \\[6pt]
\sinkState & \sinkState & 1 \\[6pt]
\text{Otherwise} & & 0
\end{array}
\]

Where $N_h(s,a,s')$ is the number of times at phase $h$ an agent was at timestep $h$, state $s$, took action $a$, and reached $s'$.
The total number of visitations is denoted with $N_h(s,a) = \sum_{s'\in\calS} N_h(s,a,s')$.

\subsubsection{The dynamics $\tP$}

We define an intermediate dynamics, denoted by $\tP$, which serves as a bridge for comparing the empirical estimates to the true dynamics. This model is constructed timeste-by-timestep by restricting transitions from states that are empirically $\beta$-reachable. 

Specifically, we continue the dynamics only from states that meet a reachability threshold based on our current estimations. Since this set depends on estimation, it is a random variable.

To account for the probability mass that enters non-reachable states, we introduce a sink state $\sinkState$. The state space of $\tP$ is $\calS \cup \{\sinkState\}$.

% \begin{definition}
% The occupancy measure of $\tP$ is denoted with $\tq$, where we exclude $\sinkState$.
% Note, that exactly like $\hq$, the vector $\tq_h \in \mathbb{R}^{|\calS|}$ does not constitute a full probability distribution over the augmented state space, as it specifically excludes the probability mass assigned to the sink state $\sinkState$.
% \end{definition}

\subsubsection{The transition function of $\tP$}
The \textbf{transition function} $\tP_h$ is defined by:

\[
\begin{array}{c c | c}
s & s' & \tP_h(s' \mid s, a) \\
\hline
\rule{0pt}{15pt}\tS_h & \calS & P_h(s' \mid s, a) \\[6pt]
\calS \setminus \tS_h & \sinkState & 1 \\[6pt]
\sinkState & \sinkState & 1 \\[6pt]
\text{Otherwise} & & 0
\end{array}
\]

I.e., the only change from $\hP$ is that $\tP$ uses the true dynamics from $\tS$ to $\calS$, while $\hP$ is the empirical mean.

We denote \textbf{occupancy measure} of this dynamics as $\tq_h(s \mid \pi) := q_h(s \mid \pi, \tP)$. Similar to $\hq$, it is over only states in $\calS$, and excludes $\sinkState$.

\subsubsection{Shared dynamics Definitions and Notations}

% \subsubsection{Notations of phases and timesteps}

\begin{remark}
\label{rem:occ-meas-formal}
At learning phase $i$, our algorithm (defined later) builds dynamics $\hP_i$.
The occupancy measure at the beginning of phase $i$ is $q_{h}(s \mid \pi, \{\hP_j\}_{j=0}^{i-1})$, for $h\leq i$.
For simplicity, we write $q_{h}(s \mid \pi, \hP)$, instead of $q_{h}(s \mid \pi, \{\hP_j\}_{j=0}^{h-1})$, when it is clear from the context.
% The \textbf{formal writing of the occupancy measures with respect to phases is as follows}.
% The occupancy measure the algorithm builds at the beginning of phase $i$, for timestep $h$, is: $q_{h}(s \mid \pi, \{\hP_k\}_{k=0}^i)$.
% For simplicity, we write $q_{h}(s \mid \pi, \hP)$, and we omit notations when they are clear from the context, e.g., $\hq_h(s\mid \pi)$ or even $\hq_h$.
% We also omit notations for other measures, $\tq$, that we defined, and $\bq$, which we will define later.

% I.e., just for this remark, let the calculated occupancy measure of timestep $h$ at phase $j$ be $q^{(j)}_{h}(s \mid \pi, \hP)$.
% We get that, $\forall i$ s.t. $i\leq h$ it holds that $q^{(h)}_{i}(s \mid \pi, \hP) = q^{(i)}_{i}(s \mid \pi, \hP)$. It holds for $q_h(\cdot \mid \pi, \tP)$ as well, that depends on $\tS$ that does not change after it was already updated.
% Hence,
% we write $\hq_h$, and sometimes we refer to a phase and actually make a claim only for the specific changed timestep.
\end{remark}

\begin{definition}
\label{def: transition matrix}
    The states \textbf{transition matrix} from $\calS$ to $\calS$ at timestep $h$ of dynamics $P$
    is
    $P^{\pi}_h[s,s'] = \sum_{a\in\calA}P_h(s'\mid s, a)\pi_h(a\mid s)$.

    % For the estimated dynamics $\hP$, given a policy $\pi$, we denote it by $\hP_h^\pi$, or simply $\hP_h$, where it is clear what the policy is.

    % Notice that for states in $\calS\setminus\tS_h$, the rows in $\hP$ are zero.
    % We can write the occupancy measure with the transition matrix:
    % \[\hq_{h+1} = \hq_h \hP_h. \]
    % Notice that $q_h(\cdot \mid \pi, \hP)$ is a row vector in our notations.

    The transition matrices for $\realP_h,\tP_h,\bP_h$ are $\realP_h^\pi,\tPpolicy{\pi}_h,\bPpolicy{\pi}_h$ accordingly.
\end{definition}

\begin{remark}
    Note that the norm over the occupancy measures, $\normone{q_h(\cdot \mid \pi, \hP)}$ etc., is just over the states in $\calS$, since we excluded $\sinkState$ from the $q_h(\cdot \mid \pi, \hP),q_h(\cdot \mid \pi, \tP),\bq$ vectors.
\end{remark}

\begin{table}[t]
\centering
\begin{tabular}{lll}
\hline
\textbf{Notation} & \textbf{Meaning}\\[3pt]
\hline
$\realP$ 
& True dynamics\\[3pt]

$\hP$ 
& Estimated dynamics; Estimated $\beta$-reachable states\\[3pt]

$\tP$ 
& True dynamics; Estimated $\beta$-reachable states\\[3pt]

$\bP$ 
& True dynamics; True $\morebeta$-reachable states \\[3pt]

\hline
\end{tabular}
\caption{Summary of dynamics notations.}
\label{tab:beta-mdp-notation}
\end{table}

\subsection{The Good Event}

\begin{definition}
\label{def:good-event-h}

    The good event of timestep $h$, denoted by $E_h$, is the event in which
    simultaneously for every state $s\in\tS_h$ and action $a\in\calA$ it holds that

   \[
   \normone{\hP_h(\cdot \mid s, a ) - \tP_h(\cdot\mid s, a)} \leq 
        \sqrt{\frac{\log(1/\delta')+2S}{\frac{m}{SA}q_h(s\mid \widehat{\pi}^{h,s}, P^1)}}.
   \]
\end{definition}

\begin{definition}
\label{def:good-event-up-to-h}
    We denote the good event up until $h$ to be $E_{\leq h} = \cap_{i\leq h} E_i$.
    The good event is where every $E_h$ happens, i.e., $E_{\leq H-1}$.
\end{definition}

From now on, we will assume that $E_{\leq h}$ holds.
Later we will show that it holds w.h.p. for every $h$, i.e., the good event holds w.h.p.

\begin{definition}
\label{def:delta-tag}
    We define $\delta'$ to be,
    \[
    \delta' := \frac{\delta}{\operatorname{supp}(m)}\cdot \frac{1}{SHA},
    \]
    where $\operatorname{supp}(m)$ is an upper bound on $m$.
    Note that $\delta'$ depends on $m$ as well.
    See below \Cref{rem: delta and m} for more details.
\end{definition}

\begin{remark}
    \label{rem: delta and m}
    In the concentration bounds we use $\delta'$ and not directly $\delta$, since the confidence intervals we use contain random variables.
    That said, we bound the number of agents $m$ with a term that contains $m$.
    It works since what we means is that one should take a large enough $m$ that satisfies it.
    Large enough $m$ always exists, as roughly speaking we bound $m$ from below with $\log(m)$ (i.e., we need to find $m$ s.t. $m\gtrsim \log(m)$).
    Note that satisfying this condition changes the value only by constants, and it does not change the parameters' powers, i.e., the power of $S,H,A$ and $\epsilon$, or the $\log$ powers.
    One can see $\log(1/\delta')$ as of order of $\log(SHA/(\delta\epsilon))$.
    
    % For our upper bound, $\agentsNumberEpsilon$, one can take $\delta'=(\frac{\delta}{SHA}\cdot (89\frac{S^5 H^6 A(\log(SHA/\delta) + 2S)}{\epsilon^2})^{-2})$.
    % Then,
    % \[(89\frac{S^5 H^6 A(\log(SHA/\delta) + 2S)}{\epsilon^2})^2\geq 89(\frac{S^5 H^6 A (2\log(SHA/(\delta89\frac{S^5 H^6 A(\log(SHA/\delta) + 2S)}{\epsilon^2}) + 2S)}{\epsilon^2}).\]
    
    % Since in our bounds we can choose $m$ larger than $2$ and then $m^2\geq 2m$.
\end{remark}

\subsection{Bounding $\hP$ with $\tP$ dynamics}

\begin{lemma}
    \label{lem: calculated policy is good for beta reachable states}
    Let $s\in\tS_h$.
    Assume that for each policy $\pi$
    \[
    \normone{q_h(\cdot \mid \pi , \hP)- q_h(\cdot \mid \pi , \tP)} \leq \frac{\beta}{3}.
    \]
    Then,
    \[
           \frac{2}{3} \beta \leq q_h(s\mid \widehat{\pi}^{h,s} , \tP).
    \]
\end{lemma}

\begin{proof}

    \begin{align*}
        q_h(s\mid \widehat{\pi}^{h,s} , \tP)
        \geq 
        q_h(s\mid \widehat{\pi}^{h,s} , \hP) - \frac{\beta}{3}
        \geq \beta - \frac{\beta}{3} = \frac{2}{3}\beta,
    \end{align*}
    where the first inequality is from the assumption, and the second since $s\in\tS_h$.
   % \ourtodo{we may change it to $2/3$ where we use it.}
\end{proof}

\begin{lemma}
\label{lem: occupancy measure mult confidence interval is small}
    The bound on the occupancy measure multiplied by the confidence interval is bounded by $\alpha/S$.
    Specifically,
    Assume $\beta/3\leq q_h(s\mid \hBestArrivingPolicy, \tP)$.
    Assume
    \[
    m \geq \max\{\frac{4 S^3 A (\log(1/\delta')+2S)}{\alpha^2}, \frac{4 H S^3 A(\log(1/\delta')+2S)}{\alpha}\}
    \]
    Then,
    \[
    (q_h(s\mid \hBestArrivingPolicy, \tP) +\alpha h) \sqrt{\frac{\log(1/\delta')+2S}{\frac{m}{SA}  q_h(s \mid \widehat{\pi}^{h,s}, \tP)}} \leq \frac{\alpha}{S}
    \]
\end{lemma}

\begin{proof}
    First, the term involving $q_h(s\mid \widehat{\pi}^{h,s} , \tP)$ satisfies,
\begin{align*}
    & q_h(s \mid \widehat{\pi}^{h,s}, \tP) \cdot
    \sqrt{\frac{\log(1/\delta')+2S}{ \frac{m}{SA}  q_h(s \mid \widehat{\pi}^{h,s}, \tP)}}
    \\&=
     \sqrt{\frac{S A(\log(1/\delta')+2S)}{m}}
    \sqrt{q_h(s\mid \widehat{\pi}^{h,s} , \tP)}
    \\&\leq
    \sqrt{\frac{S A(\log(1/\delta')+2S)}{m}}
    \quad \text{(since $q_h(s\mid \widehat{\pi}^{h,s} , \tP)\leq 1$)}
    \\&\leq \frac{\alpha}{2 S},
\end{align*}
And it holds from the assumption on a large enough $m$.
Specifically,
\begin{align*}
  \sqrt{\frac{S A(\log(1/\delta')+2S)}{m}} & \leq \frac{\alpha}{2 S}
 \\ \frac{S A(\log(1/\delta')+2S)}{m} & \leq \frac{\alpha^2}{4 S^2},
\end{align*}
I.e.,
\begin{equation}
\label{eq: m bounds the main term}
    m \geq \frac{ 4S^3 A (\log(1/\delta')+2S)}{\alpha^2}.
\end{equation}

Next, the term involving $\alpha h$ is bounded as well,
\begin{align*}
    &\alpha h \sqrt{\frac{\log(1/\delta')+2S}{\frac{m}{SA}  q_h(s \mid \widehat{\pi}^{h,s}, \tP)}}
    \\&\leq \alpha h
    \sqrt{\frac{\log(1/\delta')+2S}{ \frac{m}{SA} \cdot \frac{1}{3}\beta }} \quad \text{(from the assumption)}
    \\& \leq \alpha H
    \sqrt{\frac{S A(\log(1/\delta')+2S)}{  m}}
    \frac{1}{\sqrt{\frac{1}{3}3 \alpha H}}
    \\& = \sqrt{\frac{S A(\log(1/\delta')+2S) \alpha H}{m}}
    \\&\leq \frac{\alpha}{2 S},
\end{align*}
where the last inequality is since we need the following equation to hold:
\begin{align*}
    \sqrt{\frac{S A(\log(1/\delta')+2S) \alpha H}{m}}&\leq \frac{\alpha}{2 S}
    \\ \frac{S A(\log(1/\delta')+2S) \alpha H}{m}&\leq \frac{\alpha^2}{4 S^2}
\end{align*}
And it holds from the assumption on a large enough $m$:
\begin{equation}
    \label{eq: m bound the beta term}
    m \geq \frac{4 H S^3 A(\log(1/\delta')+2S) H}{\alpha}.
\end{equation}
\end{proof}

\begin{definition}
    Denote the bound on the number of agents with
    \[
    m(S,A,H,\alpha,\delta') :=  \max\{\frac{4 S^3 A (\log(1/\delta')+2S)}{\alpha^2}, \frac{4 H S^3 A(\log(1/\delta')+2S)}{\alpha}\}.
    \]
\end{definition}

\begin{lemma}
\label{lem: unaware good estimations}

Assume we estimated $\tP$ until and include timestep $h$ with \algacronyms{} algorithm, where the number of agents satisfies,
\[
    m \geq m(S,A,H,\alpha,\delta').
\]
Let $\pi$ be a policy.

Assume that,
\[
\normone{q_h(\cdot \mid \pi, \hP) - q_h(\cdot \mid \pi, \tP)} \leq \alpha h.
\]

Then,
\[
\normone{q_h(\cdot \mid \pi, \hP)(\hP^\pi_h-\tPpolicy{\pi}_{h})} \leq \alpha.
\]

% where we omit the policy notation for simplicity. I.e., $\tq_h(s) := q_h(s\mid \pi , \tP)$, and $\tP_h[s,s'] := \tP_h(s'\mid  s, \pi_h(s))$, and similarly for $q_h(\cdot \mid \pi, \hP)$ and $\hP$.

\end{lemma}
\begin{proof}

    From the assumption, and since $\beta =  3 \alpha H$,
    for every policy $\pi'$ it holds that $\abs{q_h(s\mid \pi' , \tP)- q_h(s\mid \pi', \hP)} \leq \beta /3,$
    % In particular for $\pi^{h,s}$ and $\widehat{\pi}^{h,s}$.
    Then the conditions of $\Cref{lem: calculated policy is good for beta reachable states}$ hold, and we get that for every $s\in\tS_h$,
    \begin{equation}
        \label{eq: occ is not small}
       \frac{2}{3} \beta \leq q_h(s\mid \widehat{\pi}^{h,s} , \tP).
    \end{equation}

    We now show that our confidence intervals are small enough.
    
    Let $s \in \tS_{h}$.

    From the good event, 
        \begin{equation}
    \label{eq: bretagnolle huber carole with m}
    \normone{\hP_h(\cdot \mid s, a ) - \tP_h(\cdot\mid s, a)} \leq \sqrt{\frac{\log(1/\delta')+2S}{\frac{m}{SA}  q_h(s \mid \widehat{\pi}^{h,s}, \tP)}}.
    \end{equation}

    Hence,
    \begin{align*}
        & \sum_{s'\in \calS}q_h(s\mid \pi, \hP) \abs{\hP_h(s'\mid s, \pi_h(s)) - \tP_h(s'\mid s, \pi_h(s))}
        \\& = q_h(s\mid \pi, \hP) \sum_{s'\in \calS} \abs{\hP_h(s'\mid s, \pi_h(s)) - \tP_h(s'\mid s, \pi_h(s))}
        \\& \leq q_h(s\mid \widehat{\pi}^{h,s} , \hP) \sum_{s'\in \calS} \abs{\hP_h(s'\mid s, \pi_h(s)) - \tP_h(s'\mid s, \pi_h(s))} \quad \text{\quad (by def. of $\hBestArrivingPolicy$) }
        \\& \leq (q_h(s\mid \hBestArrivingPolicy, \tP) + \alpha h) \abs{\hP_h(s'\mid s, \pi_h(s)) - \tP_h(s'\mid s, \pi_h(s))} \quad \text{(since $\norminf{\cdot}\leq \normone{\cdot}$)}
        \\& \leq (q_h(s\mid \hBestArrivingPolicy, \tP) +\alpha h) \sqrt{\frac{\log(1/\delta')+2S}{ \frac{m}{SA}  q_h(s \mid \widehat{\pi}^{h,s}, \tP)}}. \quad \text{(from the good event, \Cref{eq: bretagnolle huber carole with m}) }
    \end{align*}

From \Cref{eq: occ is not small}, the condition of $\beta/3\leq q_h(s\mid \hBestArrivingPolicy, \tP)$ holds. Hence,
we can use \Cref{lem: occupancy measure mult confidence interval is small}, and we get,
\[
(q_h(s\mid \hBestArrivingPolicy, \tP) +\alpha h) \sqrt{\frac{\log(1/\delta')+2S}{\frac{m}{SA}  q_h(s \mid \widehat{\pi}^{h,s}, \tP)}} \leq \frac{\alpha}{S}.
\]

Therefore, 
\begin{equation}
\label{eq: probability divergence is small under norm 1}    
q_h(s\mid \pi, \hP)\sum_{s'\in \calS}
\abs{\hP_h(s'\mid s, \pi_h(s)) - \tP_h(s'\mid s, \pi_h(s))}
\leq \frac{\alpha}{S}.
\end{equation}

Summing over all $s'\in\calS$ yields
\begin{align*}
    & \normone{\hq_{h}(\hP_{h}-\tP_{h})}
    \\& = \sum_{s'\in\calS}\sum_{s\in \tS_h}
    q_h(s\mid \pi, \hP)
    \abs{\hP_h(s'\mid s, \pi_h(s)) - \tP_h(s'\mid s, \pi_h(s))}
    \\& = \sum_{s\in \tS_h} q_h(s\mid \pi, \hP) \sum_{s'\in\calS}
    \abs{\hP_h(s'\mid s, \pi_h(s)) - \tP_h(s'\mid s, \pi_h(s))}
    \\&\leq \sum_{s'\in\calS} \frac{\alpha}{S} \quad \text{(from \Cref{eq: probability divergence is small under norm 1})}
    \\& \leq \alpha.
\end{align*}

Hence,
\[
\normone{q_h(\cdot \mid \pi, \hP)(\hP^\pi_h-\tPpolicy{\pi}_{h})} \leq \alpha,
\]
and note that the $L_1$ norm is taken over the entire state space $\calS$ and note just $\tS_h$.
\end{proof}
% Hence,
%     \begin{align*}
%         & \sqrt{\frac{m}{4 S A}q_h(s\mid \widehat{\pi}^{h,s} , \tP)}
%         \\& \geq \sqrt{\frac{m}{4 S A} \beta/4} \quad \text{(from \Cref{lem: calculated policy is good for beta reachable states}).}
%         \\& = \sqrt{\frac{m}{4 S A} 10\cdot 10\alpha }
%         \\& \geq \frac{c \cdot S ^2 H}{25} \quad \text{(for $m> c^2\cdot S^5 H^2/\alpha$)}.
%         \\& > 1.
%     \end{align*}

    % The left side of expression ($*$) is
    % \begin{align*}
    %     & \frac{q_h(s \mid \widehat{\pi}^{h,s}, \tP)}{\sqrt{\frac{m}{4 S A}q_h(s\mid \widehat{\pi}^{h,s} , \tP)}}
    %     \\& = \sqrt{\frac{q_h(s \mid \widehat{\pi}^{h,s}, \tP)}{\frac{m}{4 S A}}} \leq \frac{\alpha}{S^2 H}. \quad \text{(for $m >> SA /\alpha^2$ and since $q\leq 1$)}
    % \end{align*}
    % The right side of expression ($*$) is,
    % \begin{align*}
    %     & \frac{2\cdot \alpha h}{\sqrt{\frac{m}{S A} \frac{q_h(s\mid \widehat{\pi}^{h,s} , \tP)}{4}}}
    %     \\& \leq \frac{2\cdot \alpha h}{2 S^2 H}\leq \frac{\alpha}{S^2}.
    % \end{align*}

    % And the expression ($*$) becomes $\ll 2 \alpha$.

    % We get

    % \[
    % \norminf{q_h(s\mid \pi , \tP) (\tP_h(s' \mid s, \pi_h(s)) - \hP_h(s'\mid s, \pi_h(s))}\leq  \frac{\alpha}{S^2}.
    % \]
    % where the norm is over $s$ ($s'$ is fixed).
    % Hence,
    % \begin{align*}
    % & \sum_{s'\in S} \sum_{s\in \tS_h} \abs{q_h(s\mid \pi , \tP)} \abs{\hP_h(s'\mid s, \pi_h(s)) - \tP_h(s'\mid s, \pi_h(s))}
    % \\& \leq \sum_{s'\in S} \sum_{s\in \tS_h} \frac{\alpha}{S^2} \leq \alpha.
    % \end{align*}

    % and we get
    % \[
    % \normone{\hq_{h+1}(\hP_{h+1}-\tP_{h+1})} \leq  \alpha,
    % \]
    % where the $L1$ norm is over all the states.

\begin{lemma}
    \label{lem: h-marfe one step}
    Assume we play \algacronyms{} algorithm,
    with $m$ agents, where
    \[
    m \geq m(S,A,H,\alpha,\delta').
    \]
    Assume that for every policy $\pi$,
    \[
    \normone{q_{h}(\cdot\mid\pi,\hP)-q_h(\cdot\mid\pi, \tP)}\leq \alpha h.
    \]
    
    Then,
    \[
    \normone{q_{h+1}(\cdot\mid\pi,\hP)-q_{h+1}(\cdot\mid\pi, \tP)}\leq \alpha (h+1).
    \]
 
\end{lemma}

\begin{proof}
    Since $E_{\leq h}$ holds:
    \begin{align*}
    &\normone{ \hq_{h+1} - \tq_{h+1} }
        \\& = \normone{\hq_{h}\hP_h - \tq_{h}\tP_h}
        \\& = \normone{\hq_{h}\hP_h - \hq_h\tP_h + \hq_h\tP_h - \tq_{h}\tP_h}
        \\& \leq \normone{\hq_{h}\hP_h - \hq_h\tP_h} + \normone{\hq_h\tP_h - \tq_{h}\tP_h}
        \\& = \normone{\hq_{h}(\hP_h - \tP_h)} + \normone{(\hq_h- \tq_{h})\tP_h}
        \\& \leq \normone{\hq_{h}(\hP_h - \tP_h)} + \normone{\hq_h- \tq_{h}} \quad \text{(from \Cref{lem: norm one of vector mult row stochastic matrix})}
        \\& \leq \normone{\hq_{h}(\hP_h - \tP_h)} + \alpha  h \quad\text{(assumption)}
        \\& \leq  \alpha + \alpha  h \quad \text{(from \Cref{lem: unaware good estimations})}
        \\& \leq \alpha  (h + 1).
    \end{align*}
        Recall we denote $\hq_h := q_h(s \mid \pi, \hP)$, and similar for $\tq_h$,
    and we also omit the policy notation in the transition matrix $\hP_h := \hPpi_h$, and similar for $\tP_h$.
    See \Cref{rem:occ-meas-formal} for the notations and abbreviations.
\end{proof}

\begin{corollary}
\label{cor: diff up h}
    As we assumed until now, we assume $E_{\leq h}$ holds (we explicitly assumed that for simplicity).
    Assume we play \algacronyms{} algorithm,
    with $m$ agents, where
    \[
    m \geq m(S,A,H,\alpha,\delta').
    \]
       
    Then,
    \[
    \normone{q_{h+1}(\cdot\mid\pi,\hP)-q_{h+1}(\cdot\mid\pi, \tP)}\leq \alpha (h+1).
    \]
\end{corollary}

\begin{proof}
    By induction.

    %%%%%%%%%%%%%%% h = 0 %%%%%%%%%%%%%%%

    For $h=0$, trivial, since $\hq_0=\tq_0$.
    
    %%%%%%%%%%%%%%% h+1 %%%%%%%%%%%%%%%

    Assume true for $h$.
    For $h+1$ we use \Cref{lem: h-marfe one step}.
\end{proof}
\begin{lemma}
    \label{lem: h-marfe many arrive event}
    Assume we play \algacronyms{} algorithm,
    with $m$ agents, where
    \[
    m \geq m(S,A,H,\alpha,\delta').
    \]
    Assume that for every policy $\pi$,
    \[
    \normone{q_{h}(\cdot\mid\pi,\hP)-q_h(\cdot\mid\pi, \tP)}\leq \alpha h.
    \]

    Then for every state in $\tS_h$ and every action simultaneously 
        \[
    \prob{N_h(s,a) \leq \frac{1}{2}\frac{m}{SA} q_h(s\mid \widehat{\pi}^{h,s},P^1)}\leq \delta'/(2H).
    \]
\end{lemma}

\begin{proof}
    From \Cref{lem: calculated policy is good for beta reachable states}, for $s\in \tS_h$,
    $q_h(s\mid \widehat{\pi}^{h,s},P^1)\geq \alpha H$.

    \begin{align*}
    &\prob{N_h(s,a) \leq \frac{1}{2}\frac{m}{SA} q_h(s\mid \widehat{\pi}^{h,s},P^1)}
    \\&\leq \prob{N_h(s,a) \leq \frac{1}{2}\Expect{N_h(s,a)} \mid  \F_{< h} }
    \\& \leq e^{-\frac{1}{8}q_h(s\mid \widehat{\pi}^{h,s},P^1) SA / m}
    \\& \leq  e^{-\frac{1}{8}\alpha HSA / m}\leq \frac{\delta'}{2HSA}.
    \end{align*}

    $\F_{<h}$ is the filtration induced by what the agents saw up until the beginning of timestep $h$. I.e., $\F_{<0} = \Omega$.

    The first inequality is since $\Expect{N_h(s,a)\mid \F_{< h}} \geq \frac{m}{SA} q_h(s\mid \widehat{\pi}^{h,s},P^1)$ (each agent gets random state-action pair to explore).
    
    The second inequality is Chernoff bound (\Cref{lem: Chernoff Hoeffding}), and again  $\Expect{N_h(s,a)\mid \F_{< h}} \geq \frac{m}{SA} q_h(s\mid \widehat{\pi}^{h,s},P^1)$.
    
    The last inequality is since $m \geq m(S,A,H,\alpha,\delta') \geq HSA \log (2HSA/\delta)$.
    
    From union bound over all states and actions, we get the bound.
\end{proof}

We denote with $C_h$ the event that for every $s,a \in \tS_h \times \calA$ it holds that $N_h(s,a) > \frac{1}{2}\frac{m}{SA} q_h(s\mid \widehat{\pi}^{h,s},P^1)$.

We now prove that the good event holds w.h.p., and hence will not assume that $E_{\leq h}$ holds, but rather will prove that it holds w.h.p. for every $h$.

\begin{lemma}
    \label{lem: h-marfe full good event}
    Assume we play \algacronyms{} algorithm with $m$ agents, where
    \[
    m \geq \max\{m(S,A,H,\alpha,\delta')\}.
    \]
    Then, with probability at least $1 - \delta'$, the good event $E_{\leq H-1} = \bigcap_{h=0}^{H-1} E_h$ holds.
\end{lemma}
\begin{proof}
    We prove by induction on $h$ that
    \begin{equation}
        \label{eq: induction hypothesis}
        \prob{E_{\leq h}} \geq 1 - (h+1)\delta'/H.
    \end{equation}

    \textbf{Base case ($h=0$).}
    The event $E_0$ depends only on the trajectories collected in phase $0$. Conditional on the (deterministic) initialization $\F_{<0}$:
    By \Cref{lem: h-marfe many arrive event}, $C_h$ holds failure probability at most $\delta'/(2H)$.

    By \Cref{lem: bound for Bretagnolle Huber Carole with random variable} applied with $\calK = N_0(s,a)$, the BHC concentration in terms of the realized $N_0(s,a)$ holds for each fixed $(s,a)$ with failure probability at most $\delta'/(2HSA)$. A union bound over $(s,a) \in \tS_0 \times A$ gives total failure probability at most $\delta'/(2H)$.
    Since $C_0$ holds we can replace $N_0(s,a)$ with $\frac{1}{2} m/(SA) q_h(s\mid \widehat{\pi}^{h,s},P^1)$.
    I.e.,
    \[
   \normone{\hP_0(\cdot \mid s, a ) - \tP_0(\cdot\mid s, a)} \leq 
   \sqrt{\frac{\log(1/\delta')+2S}{2 N_0(s,a)}}
   \leq
        \sqrt{\frac{\log(1/\delta')+2S}{\frac{m}{SA}q_0(s\mid \widehat{\pi}^{0,s}, P^1)}},
   \]
   with probability at least $1-\delta'/(2SHA)$, and for every state (even though we have only one initial state, we keep the structure) and for every action it yields $1-\delta'/(2H)$.
    
    A union bound over the two events yields $\prob{E_0} \geq 1 - \delta'/H$.

    \textbf{Inductive step.} Suppose \Cref{eq: induction hypothesis} hold for $h$. We show it holds for $h+1$.
    
    By the chain rule,
    \[
    \prob{E_{\leq h+1}} = \prob{E_{\leq h}} \cdot \prob{E_{h+1} \mid E_{\leq h}}.
    \]
    From \Cref{cor: diff up h}, since $E_{\leq h}$ holds, the hypothesis of \Cref{lem: h-marfe many arrive event} is satisfied for $h+1$. 
    Therefore, \Cref{lem: h-marfe many arrive event} implies that  $C_{h+1}$ holds with probability at least $1-\delta'/(2H)$.
    
    Similarly to the base case: \Cref{lem: bound for Bretagnolle Huber Carole with random variable} applied with $\calK = N_{h+1}(s,a)$ conditional on $\F_{< h + 1}$ implies
    \[
   \normone{\hP_{h+1}(\cdot \mid s, a ) - \tP_{h+1}(\cdot\mid s, a)} \leq 
   \sqrt{\frac{\log(1/\delta')+2S}{2 N_{h+1}(s,a)}}
   \leq
        \sqrt{\frac{\log(1/\delta')+2S}{\frac{m}{SA}q_{h+1}(s\mid \widehat{\pi}^{{h+1},s}, P^1)}},
   \]
    with probability at least $1-\frac{\delta'}{2SHA}$.
    Union bound across all states and actions yields that this happens with probability at least $1-\frac{\delta'}{2H}$.
    
    A union bound over this event and $C_{h+1}$ yields $\prob{E_{h+1} \mid E_{\leq h}} \geq 1 - \delta'/H$. Hence
    \[
    \prob{E_{\leq h+1}} \geq \prob{E_{\leq h}}(1 - \delta'/H) \geq \prob{E_{\leq h}} - \delta'/H \geq 1 - (h+2)\delta'/H,
    \]
    proving the induction step.

    The claim is true in particular for $h = H-1$, $\prob{E_{\leq H-1}} \geq 1 - \delta'$.
\end{proof}

\begin{corollary}
    
\label{cor: similar occupancy measures for the beta estimated dynamics and the real dynamics}
    Assume the good event holds.
    Assume we play \algacronyms{} algorithm,
    with $m$ agents, where
    \[
    m \geq m(S,A,H,\alpha,\delta').
    \]
    
    Then,
    for every timestep $h$,
    for every policy $\pi$,
    
    \[
    \normone{q_{h}(\cdot\mid\pi,\hP)-q_h(\cdot\mid\pi, \tP)}\leq \alpha h.
    \]
    
\end{corollary}

\begin{proof}
    Immediate from the definition of the good event $E_{\leq H-1}$ (\Cref{def:good-event-up-to-h}) and from \Cref{cor: diff up h}.
\end{proof}

\begin{theorem}
\label{thm: beta estimated dynamics and estimated dynamics value functions are close}
For any policy and reward function, the value functions of the estimated dynamics $\hP$ and the dynamics $\tP$ are close.
Specifically:    

Let $\epsilon>0, \delta>0$.
Assume we play \algacronyms{},
with $H$ phases,
and the number of agents is at least,
\[m\geq \max\{\frac{4 S^3 A (\log(1/\delta')+2S)}{\alpha^2}, \frac{4 H S^3 A(\log(1/\delta')+2S)}{\alpha}\}.
\]

Then, with probability higher than $1-\delta$ the following holds.
For every policy $\pi$,
for every reward function $r:H\times S \times A \rightarrow [0,1]$,

\[
\abs{V^\pi_{\tP,r} - V^\pi_{\hP,r}} \leq \frac{1}{2}\alpha H^2.
\]

Also, if we also assume that $\alpha\leq \frac{2\epsilon}{H^2}$, then,

\[
\abs{V^\pi_{\tP,r} - V^\pi_{\hP,r}} \leq \epsilon.
\]

\end{theorem}

\begin{proof}
The good event holds with probability higher than $1-\delta$, see \Cref{lem: h-marfe full good event}.
Conditioning on the good event, the following holds.

From \Cref{cor: similar occupancy measures for the beta estimated dynamics and the real dynamics},
for every $h\in[H]$,
\[
\normone{q^{1,\pi}_h-\hq^{\,\pi}_h} \leq \alpha h,
\]
where the norm is over all states in $\calS$.

% We can look at the timestep times $h\in [H]$ and the states $s\in\calS$ as one set: $H\times \calS$.
% For indices simplicity, we take the set with state indices rather the actual states: $H\times S$.
% For timestep $h\in [H]$, state index $s\in S$, we get that the occupancy measure $q$ is $q(h\cdot s) = q_h(s)$.
% We get that 
% \begin{align*}    
% & V^\pi_{\realP,r}
% \\& = \mathbb{E}_{s \sim q}[r(s,\pi_h(s))]
% \\& = \mathbb{E}_{s \sim q}[\sum_a r(s,a)\pi_h(s)_a] \quad \text{(for policy which is a distribution)}
% \end{align*}

For the value function we get,
\begin{align*}
    & V^\pi_{\hP,r}
    = \mathbb{E}_{\hP,\pi}(\sum_{h=0}^{H-1}(r_h(s_h,a_h)))
    \\& = \sum_{h=0}^{H-1}\mathbb{E}_{\hP}\left(\sum_{a\in\calA}r_h(s_h,a)\pi_h(a \mid s_h)\right)
    \\& := \sum_{h=0}^{H-1}\mathbb{E}_{\hP}(f^{\pi}(s_h))
\end{align*}

Where the last equality is for simplicity of notations, as we denote the  expectation of the reward on $s_h$ with $f^{\pi}(s_h)$, 
\[
% f^{\pi}(s_h) := \mathbb{E}_{a_h \sim \pi_h(s')}(r_h(s',a_h)\mid s_h = s').
f^{\pi}(s_h) := \sum_{a\in\calA}r_h(s_h,a)\pi_h(a \mid s_h).
\]
Note that $0\leq f^{\pi}(s)\leq 1$.

Similarly for $q_h(\cdot \mid \pi, \tP)$.
% $V^\pi_{\tP,r} =  \sum_{h=0}^{H-1}\mathbb{E}_{\tP}(\mathbb{E}_{a_h \sim \pi_h(s')}(r_h(s',a_h)\mid s_h = s'))$.
Hence, from \Cref{lem: value function as occupancy measure distance},
\[
\abs{\mathbb{E}_{\hP}(f^{\pi}(s_h)) - \mathbb{E}_{\tP}(f^{\pi}(s_h))} \leq \normone{\hq_h- \tq_h}.
\]

Therefore,

\begin{align*}
& \abs{V^\pi_{\hP,r} - V^\pi_{\tP,r}}
\\& \leq \sum_{h=0}^{H-1}\abs{\mathbb{E}_{\hP}(f^{\pi}(s_h)) - \mathbb{E}_{\tP}(f^{\pi}(s_h))}
\\& \leq \sum_{h=0}^{H-1} 1 \cdot \normone{\hq_h - \tq_h} 
\\& \leq \sum_{h=0}^{H-1} \alpha h \quad \text{(from \Cref{cor: similar occupancy measures for the beta estimated dynamics and the real dynamics})}
\\& = \alpha \frac{H(H-1)}{2}
\\& \leq \frac{1}{2} \alpha H^2.
\end{align*}

Since $\alpha\leq \frac{2\epsilon}{H^2}$ we get,
\[
\frac{1}{2}\alpha H^2 \leq \epsilon.
\]
\end{proof}

\subsection{All four dynamics have similar value functions}
% We define it inductively, with all the states that can be reached with probability at least $\morebeta$ from the previous timestep.
% The states at each timestep are $\bS_h\cup \sinkState$, where $\sinkState$ is the dummy states, and the rest are $\morebeta$-reachable states from timestep $\bS_{h-1}$.
% We define $\bS_0 = \{\sinit\}$.
% We 

\subsubsection{The $\bP$ dynamics}

Analogous to the $\tP$ dynamics, we define the $\bP$ dynamics. This transition function serves as a theoretical bridge between the true dynamics $\realP$ and the $\tP$ dynamics $\tP$. While $\tP$ is constructed based on empirical reachability, $\bP$ is defined using the true transition dynamics but restricted to states that remain significantly reachable under the inductive constraint.

Specifically, in $\bP$, transitions continue only from states that are truly $\beta$-reachable within this restricted model. Unlike $\tS_h$, which depends on the collected data, the set of reachable states in this dynamics is deterministic.

Let $q_h(s \mid \pi, \bP) := \bq_h(s \mid \pi)$ be the \textbf{occupancy measure} of $\bP$, excluding $\sinkState$.
Note, it is not $q$-square, but just a notation.

\begin{definition}
    \label{def: reachable states in beta true dynamics}
    We define the set of inductive $\beta$-significant states at timestep $h$ as:
\[
    \bS_h := \{s \in \calS \mid \max_{\pi} q_h(s \mid \pi, \bP) \geq \morebeta \}.
\]
Note that $\bS_h$ is a deterministic set.
\end{definition}

\paragraph{Transition function for $\bP$.}

As before, we introduce a sink state $\sinkState$ to collect the probability mass of all non-reachable states. The state space of $\bP$ is $\calS \cup \{\sinkState\}$, and the transition kernel $\bP_h$ is defined as follows:

\[
\begin{array}{c c | c}
s & s' & \bP_h(s' \mid s, a) \\
\hline
\rule{0pt}{15pt} \bS_h & \calS & P_h(s' \mid s, a) \\[6pt]
\calS \setminus \bS_h & \sinkState & 1 \\[6pt]
\sinkState & \sinkState & 1 \\[6pt]
\text{Otherwise} & & 0
\end{array}
\]

By this definition, if a state $s$ is not $\beta$-reachable in $\bP$ (i.e., $s \notin \bS_h$), it acts as a termination point where all subsequent transitions lead to $\sinkState$ with probability $1$, effectively pruning the dynamics.

\begin{lemma}
\label{lem: occupancy measure is higher in beta estimated dynamics than beta dynamics}
    If $\bS_i\subseteq\tS_i$ for every $i < h$,
    then for every policy $\pi$ and state $s\in\bS_h$, $q_h(s\mid\pi,\bP)\leq q_h(s\mid\pi,\tP)$.
\end{lemma}

\begin{proof}
    By induction. Immediate for $h=0$.
    Assume true for $h$.
    For $h+1$:
    Assume $\bS_i\subseteq\tS_i$ for every $i < h+1$.
    Let $s\in\tS_{h+1}$.
    Denote $\bq_h(s\mid \pi) := q_h(s\mid \pi,\bP)$ and $\tq_h(s\mid \pi) = q_h(s\mid \pi,\tP)$.
    \begin{align*}
        & \tq_{h+1}(s\mid \pi) = \sum_{s'\in \tS_h} \tq_h(s'\mid \pi) \tP_h(s\mid s',\pi_h(s'))
        \\& = \sum_{s'\in \tS_h} \tq_h(s'\mid \pi) P_h(s\mid s',\pi_h(s'))
        \\& \geq \sum_{s'\in \bS_h} \tq_h(s'\mid \pi) P_h(s\mid s',\pi_h(s')) \quad \text{(since $\bS_h\subseteq \tS_h$)}
        \\& \geq \sum_{s'\in \bS_h} \bq_h(s'\mid \pi) P_h(s\mid s',\pi_h(s')) \quad \text{(induction hypothesis)}
        \\& = \bq_{h+1}(s\mid \pi)
    \end{align*}
\end{proof}

\begin{lemma}
\label{lem: beta estimated dynamics states contains beta dynamics states}
    % Assume we play \algacronyms{}.
    Assume
    \[
    m \geq m(S,A,H,\alpha,\delta').
    \]
    Then for every phase $h\in\{0,\dots, H\}$,
    it holds that $\bS_h\subseteq \tS_h$.
\end{lemma}

\begin{proof}
    Denote $\bq_h(s\mid \pi) := q_h(s\mid \pi,\bP)$ and $\tq_h(s\mid \pi) = q_h(s\mid \pi,\tP)$.

    By induction. $h=0$ true since $\bS_0 = \tS_0 = \{\sinit\}$.
    Assume true for every $j$ s.t. $j\leq h$. For $h+1$:
    Let $s\in \bS_{h+1}$.
    From the definition, there exists $\pi$ s.t.
    \[
    \morebeta \leq \bq_{h+1}(s\mid \pi).
    \]
    
    % From the induction hypothesis, for each $i<h+1$, $\bS_i\subseteq \tS_i$.
    % Then the antecedent of the claim in \Cref{lem: occupancy measure is higher in beta estimated dynamics than beta dynamics} holds for $h$, and
    Hence, from the induction hypothesis and from \Cref{lem: occupancy measure is higher in beta estimated dynamics than beta dynamics},
    \[
    \bq_{h+1}(s\mid \pi) = \sum_{s'\in\bS_{h}}\bq_h(s'\mid\pi)P_h(s\mid s',\pi) \leq \sum_{s'\in\tS_{h}}\bq_h(s'\mid\pi)P_h(s\mid s',\pi) \leq \sum_{s'\in\tS_{h}}\tq_h(s'\mid\pi)P_h(s\mid s',\pi),
    \]

    where the first inequality is from the induction hypothesis and the second from \Cref{lem: occupancy measure is higher in beta estimated dynamics than beta dynamics}.
    
    In each timestep, from \Cref{cor: similar occupancy measures for the beta estimated dynamics and the real dynamics}, we still have good estimation for all states, whether they are in $\tS_h$ or not. Therefore,
    % the following first inequality hold,
    \[
    \sum_{s'\in\tS_{h}}\tq_h(s'\mid\pi)P_h(s\mid s',\pi) = \tq_{h+1}(s\mid \pi) \leq \hq_{h+1}(s\mid\pi) + \frac{\beta}{3} \leq \hq_{h+1}(s\mid\widehat{\pi}^{h,s}) + \frac{\beta}{3},
    \]
    where the last inequality is by the definition of $\widehat{\pi}^{h,s}$.

    Putting all the pieces together,
    \[
    \hq_{h+1}(s\mid\widehat{\pi}^{h,s}) + \frac{\beta}{3} \geq \bq_{h+1}(s\mid \pi) \geq \morebeta.
    \]
    
    Hence,
    \[
    \hq_{h+1}(s\mid\widehat{\pi}^{h,s}) \geq 1\frac{2}{3}\beta \geq \beta,
    \]
    
    and $s\in\tS_{h+1}$.
\end{proof}

\begin{corollary}
\label{cor: beta estimated dynamics has larger value function than beta dynamics}
Assume
\[
m\geq m(S,A,H,\alpha,\delta').
\]
    For every policy $\pi$,
    For every reward function $r:H\times S \times A \rightarrow [0,1]$,
    it holds that,
    \[V^\pi(s_0\mid \tP,r)\geq V^\pi(s_0\mid \bP,r).
    \]
\end{corollary}

\begin{proof}
    From \Cref{lem: beta estimated dynamics states contains beta dynamics states} and \Cref{lem: occupancy measure is higher in beta estimated dynamics than beta dynamics},
    \begin{align*}
        & V^\pi(\sinit\mid \tP,r) = \sum_{h=0}^H \sum_{s\in \tS_h} \sum_{a\in \calA} q_h(s\mid \pi , \tP) r_h(s,a)\cdot\pi_h(a \mid s)
        \\& \geq \sum_{h=0}^H \sum_{s\in \bS_h} \sum_{a\in \calA} q_h(s\mid \pi , \tP) r_h(s,a)\cdot\pi_h(a \mid s)
        \\& \geq \sum_{h=0}^H \sum_{s\in \bS_h} \sum_{a\in \calA} q_h(s\mid \pi, \bP) r_h(s,a)\cdot\pi_h(a \mid s)
        \\& = V^\pi(\sinit\mid \bP,r),
    \end{align*}
    
    where the first inequality is from \Cref{lem: beta estimated dynamics states contains beta dynamics states} and the second from \Cref{lem: occupancy measure is higher in beta estimated dynamics than beta dynamics}.
\end{proof}

\begin{lemma}
\label{lem: real dynamics has larger value function than beta estimated dynamics}

    For every policy $\pi$,
    For every reward function $r:H\times S \times A \rightarrow [0,1]$, it holds that
    $V^\pi_h(s\mid \realP,r)\geq V^\pi_h(s\mid \tP,r)$.
\end{lemma}

\begin{proof}
    Similar to the proof \Cref{cor: beta estimated dynamics has larger value function than beta dynamics}, but here we use the definitions directly and we do not need the lemmas. Here we have $\tS_h\subseteq\calS$, and since $\tP$ is just a truncation of $\realP$ it holds that $q_h(s\mid \pi , \tP)\leq q_h(s\mid \pi, \realP)$.
    \begin{align*}
        & V^\pi(\sinit\mid \realP,r)
        = \sum_{h=0}^H \sum_{s\in \calS} \sum_{a\in \calA} q_h(s\mid \pi, \realP) r_h(s,a)\cdot\pi_h(a \mid s)
        \\& \geq \sum_{h=0}^H \sum_{s\in \tS_h} \sum_{a\in \calA} q_h(s\mid \pi , \tP) r_h(s,a)\cdot\pi_h(a \mid s)
        \\& = V^\pi(\sinit\mid \tP,r).
    \end{align*}
\end{proof}

% \begin{lemma}
%     \[
%     \prob{\Omega\setminus G^{\morebeta}} \leq H S \morebeta.
%     \]
%     At every timestep $h$ the states that are not in $\bS_h$ contribute to the probability mass no more than $\morebeta S$.
    
% \end{lemma}

% \begin{proof}

% Let $G_h$ be the event that we were only in states in $\bS_i$ for $i\leq h$.
% I.e., $G^\morebeta := G_H$.
% We get,
% \[
% \prob{G_h \mid G_{h-1}} \geq 1 - S \morebeta.
% \]
% Hence,
% \[
% \prob{\Omega\setminus G_h} \leq S\morebeta.
% \]
% \[
%  = \prob{G_1 \cap G_2 \dots G_H} = \prob{G_1 \cap G_2 \dots G_H}
% \]
% \[
% \prob{\Omega\setminus G^{\morebeta}} \leq 
% \]
%     For $h=1$, true immediately.
%     Assume true for $h$. For $h+1$:
%     \[
%     \prob{}
%     \]
% \end{proof}

\begin{lemma}
    \label{lem: the real dynamics value function is bounded by beta dynamics value function and a bit more}
    The value function under the real dynamics $\realP$ is bounded from above by the value function under $\bP$ and an expression that depends on $\beta$.
    
    Specifically,
    For every policy $\pi$,
    For every reward function $r:H\times S \times A \rightarrow [0,1]$,
    \[
    V^\pi(\sinit\mid \realP,r) \leq V^\pi(\sinit\mid \bP,r) + \morebeta  H^2 S.
    \]
    
\end{lemma}

\begin{proof}

    Denote with $G^{\morebeta}$ the event in which we were only at $\bP$. I.e., in every timestep $h$ we visit only states that are in $\bS_h$.
    
    From the definition of $\bS$, in every timestep $h$, we remove at most $S\morebeta$ probability mass. Hence,
    \begin{equation}
    \label{eq: bounding the probability of the more beta event}
        \Pr(\Omega \setminus G^{\morebeta}) \leq HS \morebeta.
    \end{equation}
    
    Therefore,
    \[
    \Pr(\Omega\setminus G^{\morebeta}) \mathbb{E} (\sum_h r_h(s_h,\pi_h(s_h)) \mid \Omega\setminus G^{\morebeta} ) \leq \Pr(\Omega\setminus G^{\morebeta}) H \leq H^2 S \morebeta.
    \]
    where the first inequality is since the rewards are bounded by $1$, and the second is from $\Cref{eq: bounding the probability of the more beta event}$.

    Putting both cases together, we get:
    \begin{align*}
        & V^\pi(\sinit\mid \realP,r)
        \\& \Expect{\sum_h  r_h(s_h,a_h)\mid G^\morebeta}\prob{G^\morebeta}
        + \Expect{\sum_h  r_h(s_h,a_h)\mid \Omega\setminus G^\morebeta}\prob{\Omega\setminus G^\morebeta}
        \\& V^\pi(\sinit\mid \bP,r)\prob{G^\morebeta}
        + H^2 S\cdot \morebeta.
    \end{align*}

\end{proof}

\begin{corollary}
    \label{cor: the beta estimated dynamics has similar value function as the real dynamics}
    The dynamics $\tP$ and the real dynamics $\realP$ have similar value functions.
    Formally, assume
    \[
m\geq m(S,A,H,\alpha,\delta').
\]

Then,
    \[
    V^\pi(\sinit\mid \tP,r) \leq V^\pi(\sinit\mid \realP,r) \leq V^\pi(\sinit\mid \tP,r) +  H^2 S \morebeta.
    \]
    And for $\alpha\leq \epsilon/ (6 H^3 S)$, we get,  
    \[
    V^\pi(\sinit\mid \tP,r) \leq V^\pi(\sinit\mid \realP,r) \leq V^\pi(\sinit\mid \tP,r) + \epsilon.
    \]
\end{corollary}

\begin{proof}
    We bound the value function of $\tP$ with $\realP$, $V^\pi(\sinit\mid \tP,r) \leq V^\pi(\sinit\mid \realP,r)$, from \Cref{lem: real dynamics has larger value function than beta estimated dynamics}.
    From \Cref{lem: the real dynamics value function is bounded by beta dynamics value function and a bit more} we can bound the value function of $\realP$ with $\bP$,
    \[
    V^\pi(\sinit\mid \realP,r) \leq V^\pi(\sinit\mid \bP,r) +  H^2 S \morebeta.
    \]
    From \Cref{cor: beta estimated dynamics has larger value function than beta dynamics} we bound the value function of $\bP$ with $\tP$,
    \[
    V^\pi(\sinit\mid \bP,r) +  H^2 S \morebeta \leq V^\pi(\sinit\mid \tP,r) +  H^2 S \morebeta.
    \]
    Together, we have the first result.
    For $\alpha\leq \epsilon/ (6 H^3 S)$ we get,
    \[
    H^2 S 2 \beta = H^2 S 2\cdot 3 \alpha H \leq \epsilon.
    \]
\end{proof}

\begin{lemma}
\label{lem: calc m}
    Assume 
    \[
    m\geq \max\{\frac{4 S^3 A (\log(1/\delta')+2S)}{\alpha^2}, \frac{4 H S^3 A(\log(1/\delta')+2S)}{\alpha}\},
    \]
    then \Cref{cor: the beta estimated dynamics has similar value function as the real dynamics} and \Cref{thm: beta estimated dynamics and estimated dynamics value functions are close} hold.
\end{lemma}

\begin{proof}
    \Cref{thm: beta estimated dynamics and estimated dynamics value functions are close} requires
    \[
        \alpha \leq \frac{2\epsilon}{H^2}.
    \]
    \Cref{cor: the beta estimated dynamics has similar value function as the real dynamics} requires
    \[
    \alpha \leq  \epsilon/ (6 H^3 S).
    \]

    Both requirements on $\alpha$ hold when
    \[
    \alpha \leq \frac{\epsilon}{6 H^3 S}.
    \]

\Cref{lem: occupancy measure mult confidence interval is small}  requires
    \[
    m \geq m(S,A,H,\alpha,\delta') := \max\{\frac{4 S^3 A (\log(1/\delta')+2S)}{\alpha^2}, \frac{4 H S^3 A(\log(1/\delta')+2S)}{\alpha}\}.
    \]

    Hence, all the requirements are fulfilled when,
    \begin{align*}
        & m \geq \max\{\frac{4 S^3 A (\log(1/\delta')+2S)}{(\frac{\epsilon}{6 H^3 S})^2}, \frac{4 H S^3 A(\log(1/\delta')+2S)}{\frac{\epsilon}{6 H^3 S}}\}
        \\& = \max\{144\frac{S^5 H^6 A (\log(1/\delta')+2S)}{\epsilon^2}, 24\frac{ S^4 H^4 A(\log(1/\delta')+2S)}{\epsilon}\},
    \end{align*}
    
    I.e.,
    \begin{equation}
        \label{eq: number of agents to learn at H phases}
        m \geq \agentsNumberEpsilon.
    \end{equation}
\end{proof}

\begin{corollary}
    
\label{cor: real dynamics and estimated dynamics value functions are close}
The dynamics $\hP$ and $\realP$ are close.
Specifically:
Let $\epsilon>0, \delta>0$.
Assume we play \algacronyms{},
with $H$ phases,
and 
\[m\geq \agentsNumberEpsilon.
\]

Then, with probability higher than $1-\delta$,
for every stochastic Markovian policy $\pi$,
for every reward function $r:H\times S \times A \rightarrow [0,1]$,
for every dynamics $\realP$ with horizon $H$,
\[
\abs{V^\pi_{\realP,r} - V^\pi_{\hP,r}} \leq 2\epsilon
\]

\end{corollary}

\begin{proof}

From \Cref{lem: calc m}, $m$ is large enough, hence,
\Cref{cor: the beta estimated dynamics has similar value function as the real dynamics}, and \Cref{thm: beta estimated dynamics and estimated dynamics value functions are close} hold.
Therefore,
    \begin{align*}
        & \abs{V^\pi_{\realP,r} - V^\pi_{\hP,r}}
        \\& = \abs{V^\pi_{\realP,r} - V^\pi_{\tP,r} + V^\pi_{\tP,r} - V^\pi_{\hP,r}}
        \\& \leq \abs{V^\pi_{\realP,r} - V^\pi_{\tP,r}} 
        +\abs{V^\pi_{\tP,r} - V^\pi_{\hP,r}}
        \\& \leq \epsilon
        +\abs{V^\pi_{\tP,r} - V^\pi_{\hP,r}}
        \\& \leq 2\epsilon.
    \end{align*}
    Where the second inequality follows from \Cref{cor: the beta estimated dynamics has similar value function as the real dynamics}, and the last from \Cref{thm: beta estimated dynamics and estimated dynamics value functions are close}.
\end{proof}

% \orin{give a corollary that summarizes it all - and is the formal statement of \Cref{thm:main-result}}

\begin{lemma}
\label{lemma:dynamics-val-diff-all-policy}
    Assume there are two dynamics $\realP, \realP'$, s.t.
    with probability at least $ 1-\delta$
    for every reward function $r$ and policy $\pi$ it holds that 
    $        \abs{V^\pi_{P,r} - V^\pi_{P',r}} \le \epsilon.$
    
    % Assume also $\epsilon'=\epsilon/3$.
    
    Then the following holds:    
    With probability higher than $1 - \delta$, 
    for every reward function $r$,
    \[
    V^\star_{P,r} - V^{\pi'^\star}_{P,r} \le 2\epsilon,
    \]  
    where ${\pi'^\star} \in \argmax_{\pi} V^{\pi}_{P',r} $.
\end{lemma}

\begin{proof}
    \begin{align*}
        &V^\star_{P,r} - V^\pi_{P,r}
        \\& = V^{\pi^\star}_{P,r} - V^{\pi^\star}_{\realP',r} + V^{\pi^\star}_{\realP',r} - V^{\pi'^\star}_{\realP',r} + V^{\pi'^\star}_{\realP',r} - V^\pi_{\realP',r} + V^\pi_{\realP',r}  - V^\pi_{P,r}
        \\& \le
        \epsilon + V^{\pi^\star}_{\realP',r} - V^{\pi'^\star}_{\realP',r}  +  \epsilon 
        \\& \le 2\epsilon,
    \end{align*}         
        where the first inequality is from the assumption and the second is since $\pi'^\star$ is optimal in $\realP'$.
\end{proof}

\begin{theorem}
    
\label{thm:-we-can-learn-in-h-phases}
We can learn with only $H$ phases.
Specifically:
Let $\epsilon>0, \delta>0$.
Assume we play \algacronyms{},
with $H$ phases,
and,
\[m\geq \agentsNumberEpsilon.
\]

Then, with probability higher than $1-\delta$,
for every reward function $r$,
% $r:H\times S \times A \rightarrow [0,1]$,
% for every dynamics $\realP$,
\[
V^\star_{\realP,r} - V^{\widehat{\pi}_r}_{\hP,r} \leq 4\epsilon,
\]
where $\hpi_r \in \argmax_{\pi} V^{\pi}_{\hP,r} $.

\end{theorem}

\begin{proof}
    Immediately from \Cref{cor: real dynamics and estimated dynamics value functions are close} and \Cref{lemma:dynamics-val-diff-all-policy}.
\end{proof}

%%%%%%%%%%%%%%%%%%%%%%%%%%%%%%%%%%%%%%%%%%%%%%%%%%%%%%%%%%%%

\section{Lower Bound}
\label{apndx:sec:lower-bound}

In this chapter we prove our lower bound results.

\begin{remark}
    Here we assume that, after the learning phases are completed, the algorithm does not directly output the estimated dynamics $\hP$, but instead later receives a reward function and is required to return a policy.
    This generalizes the model considered in the main text.
\end{remark}

Moreover, our policy is stochastic and, without loss of generality, we assume the reward function is deterministic, as in our upper bound; see \Cref{rem:appndx:stochastic-policy} and \Cref{rem:appndx:deterministic-reward}.

% In this chapter, we investigate the lower bound on the number of agents required to achieve PAC-learning when the interaction is restricted to \textbf{only one phase}. 

% Specifically, we constrain the interaction phase between the algorithm and the environment to a single phase. We then ask whether it is possible to maintain a sample complexity that is \textbf{polynomial} in the number of states $S$, the horizon $H$, and the action space $A$, $1/\epsilon$, $1/\delta$, while learning an arbitrary dynamics with high probability and low error.

% Our results show that this is not possible for general dynamics. Instead, we demonstrate that the number of agents must scale \textbf{exponentially} with the horizon; specifically, we show that $A^{H-1}$ agents are necessary.

\subsection{The Key-Dynamic Framework}
\label{sec:key-dynamics-framework}

The Key-Dynamic is characterized by a minimal state space and deterministic transitions that simulate a ``lock” that only one sequence of actions can open.

The \textbf{state space} $\calS$ consists of only two states:
\[
\calS = \{\sinkState, s^\star\}.
\]

The \textbf{initial state} is always the ``informative" state:
\[
s_0 = s^\star.
\]

The transition functions $P_h$ are deterministic. In state $s^\star$, there exists exactly one ``correct" action $a^\star_h$ at each timestep $h$ that remains the agent in $s^\star$. Any other action leads to the absorbing state $\sinkState$. Once in $\sinkState$, the agent remains there regardless of the actions chosen.

The \textbf{transition function} is defined as follows:
\[
\begin{array}{c c | c}
s & s' & P_h(s' \mid s, a) \\
\hline
\rule{0pt}{15pt} s^\star &  s^\star & \one( a = a^\star_h ) \\[6pt]
s^\star & \sinkState &  \one( a \neq a^\star_h ) \\[6pt]
\sinkState & s^\star & 0 \\[6pt]
\sinkState & \sinkState & 1 \\[6pt]
\end{array}
\]

\begin{definition}
We call the sequence of actions that keep the agent in $\starState$ the \textbf{key} of the dynamics.    
We also call these actions the \textbf{good actions}.
\end{definition}

Although we explore without any reward, the interesting reward function will be the reward that tells if the algorithm knows the ``key" of the Key-Dynamic.
\begin{definition}
    We denote with $\keyReward$ the following reward function.
    \[
        \keyReward_h(s,a) = \one(h=H-1, s = s^\star, a = a^\star_{H-1}).
    \]
\end{definition}

\subsubsection{Structural Implications}

Below are several structural implications.
\begin{itemize}
    \item \textbf{Action Sequences:} There are $A^H$ possible action sequences, but only one unique sequence $(a^\star_1, a^\star_2, \allowbreak \dots, a^\star_H)$ maintains the agent in $s^\star$.
    \item \textbf{Information Gain:} If an agent makes a mistake at any step $h$, it transitions to $\sinkState$ for the remainder of the episode. In this state, the agent receives no new information regarding the correct actions for steps $h+1 \dots H$.
    \item \textbf{Visual Analogy:} The dynamics can be visualized as two parallel lines of length $H$. The agent traverses the $s^\star$ line as long as it provides the ``correct key" (actions). A single error drops the agent into the $\sinkState$ line, where it remains until the end of the horizon.
    \item \textbf{dynamics Space:} There are $A^H$ distinct Key-Dynamic, each corresponding to a unique sequence of optimal actions.
\end{itemize}

% \begin{definition}
%     We define the random variable of the number of agent in the $s_0$ at $h=0$ with $N_0(s_0)$.
%     As denoted above $s_0 = s^\star$, and the number of agents is $m$, generally a constant and not a random variable, but we can extent this model and use a r.v. there. This definition will be used in the analysis, and is not part of the Key-Dynamic.
% \end{definition}

\subsection{The Key-Dynamic is learnable in one phase with $m = A^H$}

\begin{proposition}
    When the number of agents is $m=A^H$, each such Key-Dynamic is learnable in one phase.
    I.e.,
    Assume $m=A^H$.
    There exists an algorithm s.t. for every $\epsilon>0, \delta>0$, a Key-Dynamic $\keyDynamics$,
    and a reward function $r:H\times S \times A \rightarrow [0,1]$, 
    after $1$ phase of learning it holds that,
    \[
    \Pr\left( V^\star(s_0 ; \keyDynamics,r) - V^{\widehat{\pi}}(s_0 ; \keyDynamics,r) \leq \epsilon  \right) \geq 1- \delta.
    \]
    Furthermore, the algorithm is deterministic and it always outputs the best policy, i.e., $\widehat{\pi} = \pi^\star$.
\end{proposition}

\begin{proof}
    We construct a deterministic algorithm that will output the best policy.
    Let us enumerate all the possible action selections in each timestep, regardless of the states.
    We have $A^H$ such selections, since we have $A$ actions for each timestep.
    Let us enumerate the agents.
    The algorithm sends the agent of index $i$ to action selections in index $i$, regardless of the state the agent reaches.

    Since we send an agent to each action selection, one agent is sent to the selection we denote with $i^\star$, the selection that keeps in $\starState$: $(a^\star_0, a^\star_1, \dots, a^\star_{H-1})$.
    
    Since we restrict the learning to Key-Dynamic, all agents start at $s^\star$, and one agent, the one in index ${i^\star}$, stays in $\starState$.
    Therefore, we know the sequence of the actions that keep in $\starState$.
    Hence, we learned completely the transition function of the dynamics we explored, $\keyDynamics$.
    The transition function we estimated is as follows,
    \[
    \begin{array}{c c | c}
    s & s' & \hP_h(s' \mid s, a) \\
    \hline
    \rule{0pt}{15pt} s^\star &  s^\star & \one( a = a^\star_h ) \\[6pt]
    s^\star & \sinkState &  \one( a \neq a^\star_h ) \\[6pt]
    \sinkState & s^\star & 0 \\[6pt]
    \sinkState & \sinkState & 1 \\[6pt]
    \end{array}
    \]

    and $\hP = \keyDynamics$.
    
    Since the algorithm and the dynamics are deterministic, we get that for any reward function $r$, the algorithm outputs $\widehat{\pi} = \pi^\star$.
    Hence,
    \[
        \Pr\left( V^\star(s_0 ; \keyDynamics,r) - V^{\widehat{\pi}}(s_0 ; \keyDynamics,r) = 0   \right) = 1,
    \]
    and we managed to learn the Key-Dynamic in one phase.
\end{proof}

\subsection{The Expectation of Discovering the Key Drops Quickly}

\begin{definition}
    The good agents, denoted with $\goodAgents_h^i(\keyDynamics)$, are the group of agents that arrive at state $\starState$ at timestep $h$ at dynamics $\keyDynamics$ with a fixed algorithm $\learningAlgorithm$ at phase $i$.
    For the ease of notations, we drop the dynamics notations where they are clear from the context.
\end{definition}

\begin{definition}
    When the problem is a single-phase learning, we denote the group of good agents $\goodAgents_h$, without the phase notations.
\end{definition}

\begin{lemma}
\label{lem: lb expectation drops}
    For every algorithm $\learningAlgorithm$ with one learning phase,
    For any Key-Dynamic with horizon $H$, $\keyDynamics_H$,
    there exists a Key-Dynamic of horizon $H+1$, denoted with $\keyDynamics_{H+1}$, that has the same prefix of length $H$ as in $\keyDynamics_H$, s.t.
    % there exists an action $\starAction\in\calA$ s.t.
    the expected number of the good agents on the dynamics $\keyDynamics_{H+1}$ drops by $1/A$.
    I.e.,
    \[
        \Expect{\abs{\goodAgents_{H+1}}} \leq \frac{\Expect{\abs{\goodAgents_{H}}}}{A},
    \]
    where $\Expect{\abs{\goodAgents_{H}}} = \Expect{\abs{\goodAgents_{H}(\keyDynamics_H)}} = \Expect{\abs{\goodAgents_{H}(\keyDynamics_{H+1})}}$.
\end{lemma}

\begin{proof}

    Denote with $\alpha_i$ the action choice of agent $i$ at timestep $H$ (the transition to the last timestep).
    When the action in the key at the timestep $H$ is $a$, we get,    
    \[
    \Expect{\abs{\goodAgents_{H+1}}\mid \goodAgents_H}
    = \Expect{\sum_{i\in\goodAgents_H}\one(\alpha_i = a \mid \goodAgents_H)}
    = \sum_{i\in\goodAgents_H}\prob{\alpha_i = a \mid \goodAgents_H}.
    \]

    We use the probabilistic method.
    We define a uniform distribution over the actions.
    More formally, the distribution is over the dynamics of horizon $H+1$, with their prefix the same as $\keyDynamics_H$.
    there are $A$ such dynamics.
    The expectation over this distribution is denoted with  $\E_{a \sim \calA}$.
    \begin{align*}
    & \E_{a\sim \calA}\left(\sum_{i\in\goodAgents_H}\prob{\alpha_i = a \mid \goodAgents_H}\right)
    \\&=\frac{1}{A}\sum_{a\in\calA}\sum_{i\in\goodAgents_H}\prob{\alpha_i = a \mid \goodAgents_H}
    \\&=\frac{1}{A}\sum_{i\in\goodAgents_H}\sum_{a\in\calA}\prob{\alpha_i = a \mid \goodAgents_H}
    \\&=\frac{1}{A}\sum_{i\in\goodAgents_H}1 \quad \text{(since every agent chooses one action)}
    \\&= \frac{\abs{\goodAgents_H}}{A}.
    \end{align*}

    Together we get,    
    \begin{align*}
        & \E_{a \sim \calA}\left(\Expect{\abs{\goodAgents_{H+1}}}\right)
        \\& = \E_{a \sim \calA}\left(\Expect{\Expect{\abs{\goodAgents_{H+1}}\mid \goodAgents_H}}\right)
        \\& = \E_{a \sim \calA}\left(\Expect{\sum_{i\in\goodAgents_H}\prob{\alpha_i = a\mid \goodAgents_H}}\right)
        \\& = \Expect{\E_{a \sim \calA}\left(\sum_{i\in\goodAgents_H}\prob{\alpha_i = a \mid \goodAgents_H }\right)}
        \\& = \Expect{\frac{\abs{\goodAgents_H}}{A}}.
    \end{align*}

    % Note that we can switch the expectations from the Fubini–Tonelli theorem, since $\sum_{i\in\goodAgents_H}\prob{\alpha_i = a \mid \goodAgents_H}$ is integrable and non-negative. 
    
    Then, there exists an dynamics of horizon $H+1$, denoted with $\keyDynamics_{H+1}$, with its prefix as $\keyDynamics_{H}$, and the next action in the key is $\starAction\in\calA$, s.t.,
    \[
    \Expect{\abs{\goodAgents_{H+1}(\keyDynamics_{H+1})}} \leq \frac{\Expect{\abs{\goodAgents_H(\keyDynamics_{H+1})}}}{A} .
    \]
    It follows from the prefixes equality that  $\goodAgents_H(\keyDynamics_{H+1}) = \goodAgents_H(\keyDynamics_{H})$.
\end{proof}

\begin{corollary}
\label{cor: lb many phases expectation drops in every timestep}
    For every algorithm $\learningAlgorithm$ with one learning phase,
    for every horizon $H$,
    there exists a Key-Dynamic with horizon $H$ s.t. for any $h\leq H$ and $i\leq h$,
    \[
    \Expect{\abs{\goodAgents_h}} \leq \frac{\Expect{\abs{\goodAgents_{i}}}}{A^{h-i}}.
    \]
\end{corollary}

\begin{proof}
    By induction on $H$.
    For $H=0$ it is true by definition.
    Assume true for $H$ and call this dynamics $\keyDynamics_{H}$. For $H+1$:
    From \Cref{lem: lb expectation drops}, there exists an dynamics $\keyDynamics_{H+1}$ of horizon $H+1$ that satisfies $\Expect{\abs{\goodAgents_{H+1}}}\leq \Expect{\abs{\goodAgents_{H}}/A}$, and its prefix is $\keyDynamics_H$.
    Therefore, from the prefixes equality, for every $h\leq H$ and $i\leq h$ it holds that $\Expect{\abs{\goodAgents_h}} \leq \frac{\Expect{\abs{\goodAgents_{i}}}}{A^{h-i}}$. For $h=H+1$ it follows since for every $i\leq H+1$, \[
    \Expect{\abs{\goodAgents_{H+1}}}\leq \frac{\Expect{\abs{\goodAgents_{H}}}}{A}\leq \frac{\Expect{\abs{\goodAgents_{i}}}}{A^{H-i}}\cdot\frac{1}{A}\leq \frac{\Expect{\abs{\goodAgents_{i}}}}{A^{H + 1 -i}},
    \]
    where the second inequality is from the induction hypothesis.
    The dynamics that was constructed is indeed a Key-Dynamic.
\end{proof}

\subsection{Lower Bound for Multiple Phases}

In this section the algorithm has $\rho$ phases, where $\rho<H$.
Similar to one phase of learning, we construct a dynamics s.t. the expected number of good agents drops by a factor of $1/A$ in every timestep.
Here we use the $\keyReward$, which indicates if the algorithm discovers the key or not.
We will have a high probability that the algorithm does not discover the entire key, and it follows that with high probability the algorithm has high error.

% We do not have $\epsilon$ in the bound of many phases, but just for the bound with one phase, since the Key-Dynamic with the MAB in the end does not give a bound that keeps the $1/\epsilon^2$ without the dependency on $\rho$.
% It is an open question to get a bound of the form $A^{H/\rho}/\epsilon^2$.
% It does not work out of the box since if the number of agents scales with $1/\epsilon$, the algorithm uses these agent not just to estimate the MAB in the last timestep but also to discover the key.

\begin{lemma}
\label{lem: lb expectation guess with epsilon}
    The algorithm cannot do better than just guessing the rest of the key.
    Formally,
    In one phase learning problem,
    for any algorithm $\learningAlgorithm$,
    for any sequence of $H-H'$ actions $\Bar{a}_{H-H'}$,
    there exists a Key-Dynamic of horizon $H$, denoted by $\keyDynamics$,
    with the first $H-H'$ good actions as $\Bar{a}_{H-H'}$,
    such that if there are $0$ agents at $\starState$ at $h=H'$, 
    then the algorithm will have $1/A^{H'}$ error at expectation.
    I.e.,
    \[
        \mathbb{E}\left(\valueAlg{\keyDynamics}{\keyReward} \mid \noAgentsHtag\right) \leq \frac{1}{A^{H'}},
    \]
    where $\noAgentsHtag$ is the event that there are $0$ agents at $\starState$ at $h=H'$ at the learning phase, and $\E$ means the expectation over the algorithm choices.
    % as well as the dynamics transitions.
\end{lemma}

\begin{proof}
    Here we use the Probabilistic Method.
    We will use the fact that when the expectation of a random variable $X$ has the value of $\mathbb{E}(X)$, there exists a point in the probability space $\omega\in\Omega$ s.t. $X(\omega) \leq \mathbb{E}(X)$.

    All the Key-Dynamic differ only by the actions that keep the agent in $\starState$. 
    Let us fix the first $H-H'$ of the actions.
    So, for these actions we have $A^{H'}$ different Key-Dynamic with horizon of $H$. Let us enumerate them with $\keyDynamicsEnum{i}$.

    Each algorithm has an interaction stage, i.e., the stage in which the agents explore the dynamics.
    After this stage, the algorithm will decide on the policy, according to the reward function it gets.
    Let $\interaction(\keyDynamicsEnum{j})$ be the interaction of algorithm $\learningAlgorithm$ with dynamics $\keyDynamicsEnum{j}$.

    If there are zero agents at the informative state $\starState$ at timestep $H'$, the interactions are the same for every Key-Dynamic with the same prefix of $H'$ actions.
    I.e.,
    \begin{equation}
    \label{eq: same interaction}
    \interaction(\keyDynamicsEnum{i} \mid \noAgentsHtag) = \interaction(\keyDynamicsEnum{j} \mid \noAgentsHtag),
    \end{equation}
    where we remind that $\noAgentsHtag$ is the event $0$ agents in $\starState$ at timestep $h=H'$
    
    Let $\Pr(\learningAlgorithm(\interaction(\keyDynamicsEnum{j});\keyReward) = \hpi, \hpi = i)$ be the probability the algorithm $\learningAlgorithm$ chooses a policy $\hpi$ and $\hpi$ chooses actions' index $i$ after interaction $\interaction(\keyDynamicsEnum{j})$, where the reward function is $\keyReward$.
    We can simply look at the probability of the algorithm to choose sequence $i$ after the interaction: $\Pr(\learningAlgorithm(\interaction(\keyDynamicsEnum{j})) = i)$, and we omit the reward for ease of notation.

    Note that for any interaction, the algorithm chooses one sequence of the last $H'$ actions
    (with probability $1$),
    hence,
    \begin{equation}
    \label{eq: alg chooses some action}                
        \sum_{i=1}^{A^{H'}}\Pr(\learningAlgorithm(\interaction(\keyDynamicsEnum{1})) = i \mid \noAgentsHtag) = 1.
    \end{equation}

    Since the reward $\keyReward$ is built in such a way that only guessing the full key yields a reward, we get,
    \begin{align*}
        & \mathbb{E}_{\learningAlgorithm}( \mid \noAgentsHtag)
        \\& = 1\cdot \Pr(\learningAlgorithm(\interaction(\keyDynamicsEnum{i})) = i \mid \noAgentsHtag) +
        0 \cdot \Pr(\learningAlgorithm(\interaction(\keyDynamicsEnum{i})) \neq i \mid \noAgentsHtag),
    \end{align*}
    and we get,
    \begin{equation}
    \label{eq: only guess the index}
    \mathbb{E}_{\learningAlgorithm}(\valueAlg{\keyDynamicsEnum{i}}{\keyReward}) \mid \noAgentsHtag)
        = \Pr(\learningAlgorithm(\interaction(\keyDynamicsEnum{i})) = i \mid \noAgentsHtag),
    \end{equation}
    where $\keyReward_h(s,a; \keyDynamicsEnum{i})$ is the reward at dynamics $\keyDynamicsEnum{i}$ at state $s$ timestep $h$ and action $a$.

    Let us define the uniform distribution over the Key-Dynamic with such a fixed first half, and denote this distribution over the environments with $\randomMdpDist$.
    Denote the expectation over the randomness of the algorithm with $\E_{\learningAlgorithm}$.
    We get,
    \begin{align*}
        & =\E_{\keyDynamics \sim \randomMdpDist}\mathbb{E}_{\learningAlgorithm}\left(\valueAlg{\keyDynamics}{\keyReward} \mid \noAgentsHtag\right)
        \\&= \sum_{i=1}^{A^{H'}} \frac{1}{A^{H'}}  \mathbb{E}_{\learningAlgorithm}\left(\valueAlg{\keyDynamicsEnum{i}}{\keyReward} \mid \noAgentsHtag\right)
        \\&= \frac{1}{A^{H'}} \sum_{i=1}^{A^{H'}} \Pr(\learningAlgorithm(\interaction(\keyDynamicsEnum{i})) = i \mid \noAgentsHtag) \quad \text{(need to guess the key: \Cref{eq: only guess the index})}
        \\&= \frac{1}{A^{H'}} \sum_{i=1}^{A^{H'}} \Pr(\learningAlgorithm(\interaction(\keyDynamicsEnum{1})) = i \mid \noAgentsHtag) \quad \text{(same interactions: \Cref{eq: same interaction})}
        \\&= \frac{1}{A^{H'}} \cdot 1 \quad \text{(one action is being chosen: \Cref{eq: alg chooses some action})}
    \end{align*}

    Therefore, there exists an dynamics $\keyDynamics$ which its first actions are $\Bar{a}_{H-H'}$ in which the algorithm performs at expectation with at most $\frac{1}{A^{H'}}$ reward, where $r= \keyReward$.       
\end{proof}

\begin{corollary}
\label{cor: lb guess rest of the key}
    With high probability, the algorithm has high error when it does not know the key.
    Formally,
    Assume we have only one learning phase,
    Assume $A\geq 2$.
    Let $\learningAlgorithm$ be an algorithm,
    for any sequence of $H-1$ actions $\Bar{a}_{H-1}$,
    there exists a Key-Dynamic of horizon $H$, denoted by $\keyDynamics$,
    with the first $H-1$ good actions as $\Bar{a}_{H-1}$,
    s.t.,
    \[
        \prob{V^\star(s_0;\keyDynamics,\keyReward)- \valueAlg{\keyDynamics}{\keyReward} \geq 0.1 \mid \noAgentsLast} \geq \frac{4}{9}.
    \]
\end{corollary}

\begin{proof}
    From Markov inequality (see \Cref{lem: markov inequality for conditional expectation} for the conditional expectation version of Markov inequality) and from \Cref{lem: lb expectation guess with epsilon},
    \begin{align*}
    &\prob{\valueAlg{\keyDynamics}{\keyReward} \geq 0.9\mid \noAgentsLast}
    \\& \leq \Expect{\valueAlg{\keyReward}{\keyReward}\mid \noAgentsLast}\cdot \frac{10}{9}
    \\&\leq \frac{1}{A}\frac{10}{9} \leq \frac{5}{9}.
    \end{align*}
    The rest follows from that $V^\star=1$.    
\end{proof}

\begin{definition}
    We define with $Z^j_i$ the event that no agent in phases $0$ to $j-1$ reached timestep $i$ in the Key-Dynamic.
\end{definition}

\begin{corollary}
\label{cor: lb conditional many phases expectation drops in every timestep}
Fix algorithm $\learningAlgorithm$.
there exists a Key-Dynamic with horizon $H$ s.t. for any $i< h \leq H$
%    For every algorithm $\learningAlgorithm$ with one learning phase,
  %  for every $H$,
  %  there exists a Key-Dynamic with horizon $H$ s.t. for any $h\leq H$ and $i\leq h$,
    \[
    \Expect{\abs{\goodAgents^j_h}\mid Z^j_i} \leq \frac{m}{A^{h-i}}.
    \]
\end{corollary}

\begin{proof}
    % For every algorithm that does not send the maximum number of agents to the furthest place in the discovered key there is algorithm that does it and preform better.
    % Hence, w.l.o.g., we can assume the algorithm sends the agents to discover the key.

    We reduce the problem to learning with a single phase.
    Under the event $Z^j_i$, we can view the part of the dynamics of timesteps after $i$ at phase $j$ as part of the dynamics in the single-phase learning.
    In this case, the good agents in the single phase problem at timestep $i$, $\goodAgents_i$, are the same as the good agents in the $\rho$ phases problem at phase $j$ at timestep $i$, $\goodAgents^j_i$.
    The number of agents at $\starState$ at the beginning of phase $j$ is at most $m$, so $\goodAgents_i\leq m$.
    And the corollary follows from \Cref{cor: lb many phases expectation drops in every timestep}, when $\abs{\goodAgents_i} \leq m$.
\end{proof}

\begin{corollary}
\label{cor: lb many phases high prob for zero agents}
    For every algorithm $\learningAlgorithm$ with one learning phase,
    There exists a Key-Dynamic of horizon $H-1$ s.t.
    for any phase $i\in\{0,\dots,\rho-1\}$,
    the probability that there are no agents in the informative state $\starState$ at timestep $i\cdot (H-1)/\rho$ is high, conditioned there were zero agents at phase $i-1$ at $(i-1)\cdot (H-1)/\rho$.
    Specifically,
    \[
        \prob{Z^i_{(i+1)H'} \mid Z^i_{iH'}} > 1- \frac{m}{A^{H'}},
    \]
        where $H' := (H-1)/\rho$.
        We assume for simplicity that $\rho$ divides $H-1$, and a similar result without this assumption can be derived.
\end{corollary}

\begin{proof}
    Remember that $i$ starts at $0$.
    Under the event $Z^i_{iH'}$, the event that there are zero agents in timestep $(i+1)H'$ at phase $i$ and the event $Z^i_{(i+1)H'}$, are the same.
    Hence,
    \[
        \prob{\abs{\goodAgents^i_{(i+1)H'}} = 0 \mid Z^i_{iH'}} = \prob{Z^i_{(i+1)H'} \mid Z^i_{iH'}}.
    \]

    From Markov inequality (\Cref{lem: markov inequality for conditional expectation}) and from \Cref{cor: lb conditional many phases expectation drops in every timestep},
    \[
        \prob{\abs{\goodAgents^i_{(i+1)H'}} \geq 1 \mid Z^i_{iH'}}\leq \frac{m}{A^{H'}}.
    \]
    then,
    \[
    \prob{\abs{\goodAgents^i_{(i+1)H'}} < 1 \mid Z^i_{iH'}}
    = \prob{\abs{\goodAgents^i_{(i+1)H'}} = 0 \mid Z^i_{iH'}} 
    > 1 - \frac{m}{A^{H'}}.
    \]
\end{proof}

\newcommand{\discoverEvent}{\mathcal{B}}
\newcommand{\goodEventManyPhases}{\mathcal{Z}}
\begin{definition}
    We denote with $\Omega$ the probability space, and we mainly use it to define complementary events. I.e., let $G$ be an event, its complementary event is denoted with $\Omega\setminus G$.
\end{definition}

\begin{theorem}
\label{thm:lower-bound-many-phases}
    Assume the number of agents satisfies,
    \[
    m\leq \frac{A^{(H-1)/\rho}}{\rho \cdot 100}.
    \]
    Assume $\rho<H$ and assume for simplicity that $\rho$ divides $H-1$ (otherwise a similar argument with $(A^{\floor{(H-1)/\rho}})/\rho$ can be derived). Assume $A\geq 2$.
    Then,
    for any algorithm $\learningAlgorithm$ with $\rho$ phases of learning,
    there exists a  Key-Dynamic $\keyDynamics$ with $A$ actions and horizon $H$ s.t.,
    % for any $\epsilon$ s.t. $1-1/A^{H/4}\geq\epsilon>0$,
    \[
    \prob{V^\star(s_0 ; \keyDynamics, \keyReward) - V^{\widehat{\pi}}(s_0 ; \keyDynamics, \keyReward) \geq 0.1 } \geq 0.55,
    \]
    where $\widehat{\pi}$ maximizes the value function under reward $\keyReward$ and the learned dynamics.    
    
\end{theorem}

\begin{proof}
% [Proof of \Cref{thm:main-lower-bound-many-phases}]

    We split the analysis into two cases.
    One is the event in which in every phase the algorithm discovers no more than $H'$ new actions in the key, where $H' := (H-1)/\rho$.
    The other is the complement event, i.e., there exists a phase in which some agent discovers more than $H'$ new actions in the key.
    We denote the first event with $\goodEventManyPhases$.
    Note that $\goodEventManyPhases = \cap_{i=0}^{\rho-1} Z^i_{(i+1)H'}$.

    We built the dynamics from two parts.
    The first $H-1$ good actions are from the $H-1$ horizon dynamics that \Cref{cor: lb many phases high prob for zero agents} promises exist.
    The last good action is from \Cref{cor: lb guess rest of the key}.
    
    Since we use the first $H-1$ actions as the actions from the Key-Dynamic in \Cref{cor: lb many phases high prob for zero agents}, we get,
    \[
        \prob{\Omega\setminus\goodEventManyPhases}\leq \rho(\frac{m}{A^{(H-1)/\rho}}) \leq \frac{1}{100},
    \]
    where we used the union bound to bound the probability.

    % \[
    %     \prob{V^{\widehat{\pi}}(s_0 ; \keyDynamics, \keyReward) \leq \frac{1}{A^{H/4}} \mid \abs{\goodAgents_{H/2}} = 0} \geq 1-\frac{1}{A^{H/4}}.
    % \]

    % % Similar to the claim in the proof of \Cref{cor: lb conditional many phases expectation drops in every timestep},
    If none of the agents reached state $\starState$ in timestep $H-1$ in every phase, i.e., $Z^{\rho-1}_{H-1}$ holds, it is the same the single-phase learning where $\abs{\goodAgents_{H-1}} = 0$.
    % I.e., the same claims used in the proof of \Cref{cor: lb guess rest of the key} and its preceding lemmas follow immediately.
    I.e., we can use \Cref{cor: lb guess rest of the key}.
    Hence, for any sequence of $H-1$ actions we have a Key-Dynamic of horizon $H$, denoted $\keyDynamics$, s.t. its prefix key is the sequence of $H-1$ actions, and,
        \[
        \prob{V^\star(s_0;\keyDynamics,\keyReward)- \valueAlg{\keyDynamics}{\keyReward} \geq 0.1 \mid \noAgentsLast} \geq \frac{4}{9}.
    \]
    
    % The rest of the proof is identical to \Cref{lem: one phase learning no epsilon}, with $\goodEventManyPhases$ instead of $\abs{\goodAgents_{H-1}} = 0 $.
    Notice that for the entire dynamics of horizon $H$, it still holds that $\prob{\goodEventManyPhases}\geq \frac{99}{100}$.

    \begin{align*}
        & \prob{V^\star(s_0 ; \keyDynamics, \keyReward ) - V^{\widehat{\pi}}(s_0 ; \keyDynamics, \keyReward )\geq 0.1}
        \\& =
        \prob{V^\star(s_0 ; \keyDynamics, \keyReward ) - V^{\widehat{\pi}}(s_0 ; \keyDynamics, \keyReward )\geq 0.1 \mid \goodEventManyPhases } \prob{\goodEventManyPhases}
        \\& +\prob{V^\star(s_0 ; \keyDynamics, \keyReward ) - V^{\widehat{\pi}}(s_0 ; \keyDynamics, \keyReward )\geq 0.1 \mid \Omega\setminus\goodEventManyPhases }
        \prob{\Omega\setminus\goodEventManyPhases}
        \\&  \geq
        \prob{V^\star(s_0 ; \keyDynamics, \keyReward ) - V^{\widehat{\pi}}(s_0 ; \keyDynamics, \keyReward )\geq 0.1 \mid \goodEventManyPhases } \prob{\goodEventManyPhases}
         + 0
         \\& \geq 
        \frac{5}{9} \cdot\prob{\goodEventManyPhases}
          \quad \text{(from \Cref{cor: lb guess rest of the key})}
        \\& \geq \frac{5}{9} \cdot \frac{99}{100} \geq 0.55.
    \end{align*}        
\end{proof}

\section{Auxiliary Lemmas}

\begin{lemma}
\label{lem: norm one of vector mult row stochastic matrix}
    Let $P\in \R^{m,n}$ be a matrix s.t. $\forall i$, $\sum_{j=1}^n P_{ij} \leq 1$, and each element $P_{ij}\geq 0$.
    E.g., for some row $i$ an entry $P_{ij}$ is the probability to move from $i$ to $j$, and for other row $i'$ it is zero for all $j$.
    Let $v\in \R^m$.
    Then,
    \[
    \normone{vP}\leq \normone{v}.
    \]
    % Where $\normone{u}_K := \sum_{j\in \rho} \abs{ \sum_i v_i P_{ij}} $ . 
\end{lemma}
\begin{proof}
    \begin{align*}
        & \normone{vP} = \sum_j \abs{ \sum_i v_i P_{ij}} \leq \sum_j \sum_i \abs{ v_i P_{ij}}
        \\& = \sum_j \sum_i \abs{ v_i} P_{ij} \quad \text{(the elements are $\geq 0$)}
        \\& =\sum_i \sum_j \abs{ v_i} P_{ij} = \sum_i \abs{v_i} \sum_j P_{ij}
        \\& \leq \sum_i \abs{v_i} \cdot 1 \quad \text{(from the assumption)}
        \\& = \normone{v}.
    \end{align*}
\end{proof}

% \begin{lemma}[Lemma F.4 in \cite{dann2017unifying}]
%     \label{lem: dann}
%      Let $\{ X_t \}_{t=1}^T$ be a sequence of Bernoulli random variables and a filtration $\mathcal{F}_1 \subseteq \mathcal{F}_2 \subseteq...\mathcal{F}_T$ with $\mathbb{P}(X_t = 1\mid \mathcal{F}_t) = P_t$, $P_t$ is $\mathcal{F}_{t}$-measurable and $X_t$ is $\mathcal{F}_{t+1}$-measurable. Then, for all $t\in [T]$ simultaneously, with probability $1-\delta$,
%      \[
%         \sum_{k=1}^t X_k \geq \frac{1}{2}\sum_{k=1}^t P_k -\log \frac{1}{\delta}.
%      \]
     
% \end{lemma}

\begin{lemma}[Markov inequality for conditional expectation]
    \label{lem: markov inequality for conditional expectation}
    Let $(\Omega, \mathcal{F}, P)$ be a probability space. Let $X\geq 0$ be a non-negative random variable.
    Let $a, b \in \mathbb{R}, a > 0, b\geq 0$. Let $A \in \mathcal{F}$ be an event.

    Assume $E(X|A) \le b$.
    
    Then, $\Pr(X \ge a | A) \le b/a$.
\end{lemma}
\begin{proof}
    
From the definition of conditional expectation we have,
\[
\int_{A} X dP = \int_{A} E(X|A) dP,
\]
and from the assumption we get,
\[
\int_{A} E(X|A) dP \le \int_{A} b dP = b Pr(A).
\]
Together we get
\begin{equation}
    \label{eq: markov cond exepect int on a}
    \int_{A} X dP \leq  b Pr(A).
\end{equation}

Hence,
\[
Pr(X \ge a | A) = \frac{Pr(X \ge a, A)}{Pr(A)} = \frac{\int_{A \cap \{X \ge a\}} 1 dP}{Pr(A)} \le \frac{\frac{1}{a} \int_{A} X dP}{Pr(A)} \le \frac{\frac{1}{a} b Pr(A)}{Pr(A)} = \frac{b}{a},
\]

where the second inequality is from \Cref{eq: markov cond exepect int on a}, and the first inequality is since,
\[
\int_{\{X \ge a\} \cap A} 1 dP = \frac{1}{a} \int_{\{X \ge a\} \cap A} a dP \le \frac{1}{a} \int_{\{X \ge a\} \cap A} X dP \le \frac{1}{a} \int_{A} X dP.
\]
\end{proof}

\begin{lemma}[Bretagnolle--Huber--Carol, Lemma 10.13 in \cite{MannorMT-RLbook}]
\label{lem: Bretagnolle Huber Carol}
Let $X$ be a random variable taking values in $\{1, \ldots, S\}$, where $\Pr[X = i] = p_i$.
Assume we sample $X$ for $n$ times and observe the value $i$ in $\hat{n}_i$ outcomes.
Then,
\[
\Pr\!\left[
  \sum_{i=1}^{k} \left| \frac{\hat{n}_i}{n} - p_i \right| \ge \lambda
\right]
\le 2^{k+1} e^{-2n\lambda^2}.
\]
\end{lemma}

\begin{corollary}[Bretagnolle--Huber--Carol with a $2S$ term]
\label{cor: bound for Bretagnolle Huber Carole}
Let $X$ be a random variable taking values in $\{1, \ldots, S\}$, where $\Pr[X = i] = p_i$.
Assume we sample $X$ for $n$ times and observe the value $i$ in $\hat{n}_i$ outcomes.
Denote the empirical distribution with $\hat{p}$, i.e., $\hat{p}_i = \hat{n}_i/ n$
Then for any $\delta\in(0,1)$, with probability at least $1-\delta$,
\[
\|\hat p - p\|_1 \;\le\; \sqrt{\frac{\ln(1/\delta)+2S}{2n}}.
\]
\end{corollary}

\begin{proof}
The Bretagnolle--Huber--Carol inequality gives, for all $\lambda>0$,
\[
\Pr\!\left(\|\hat p-p\|_1 \ge \lambda\right)\;\le\; 2^{S+1}e^{-2n\lambda^2}.
\]
Set the right-hand side to be at most $\delta$:
\[
2^{S+1}e^{-2n\lambda^2}\le \delta
\;\Longleftrightarrow\;
2n\lambda^2 \ge \ln(1/\delta) + (S+1)\ln 2.
\]
Since $S\ge 1$, we have $(S+1)\ln 2 \le 2S$, hence it suffices that,
\[
2n\lambda^2 \ge \ln(1/\delta) + 2S,
\]
which yields the stated bound.
\end{proof}

\begin{lemma}[Bretagnolle--Huber--Carol with a random variable]
\label{lem: bound for Bretagnolle Huber Carole with random variable}
Let $\calK$ be a random variable with support of cardinality $\operatorname{supp}(\calK)$.
Let $X$ be a random variable taking values in $\{1, \ldots, S\}$, where $\Pr[X = i] = p_i$.
Assume we sample $X$ for $\calK$ times and observe the value $i$ in $\hat{\calK}_i$ outcomes.
Denote the empirical distribution with $\hat{p}$, i.e., $\hat{p}_i = \hat{\calK}_i/ \calK$
Then for any $\delta\in(0,1)$, with probability at least $1-\delta$,
\[
\|\hat p - p\|_1 \;\le\; \sqrt{\frac{\log(\operatorname{supp}(\calK)/\delta)+2S}{2\calK}}.
\]
\end{lemma}

\begin{proof}
% Given the concentration bound for the $L_1$ distance:
% \begin{equation*}
%     \Pr \left( \| \hat{p} - p \|_1 \geq \lambda \right) \leq 2^{S+1} e^{-2K\lambda^2}
% \end{equation*}

Define the threshold $\lambda_\calK$ as:
\begin{equation*}
    \lambda_\calK = \sqrt{\frac{\log(\operatorname{supp}(\calK)/\delta) + 2S}{2\calK}}.
\end{equation*}

We evaluate the probability by conditioning on the random variable $\calK$:
\begin{align*}
    \Pr \left( \| p - \hat{p} \|_1 \geq \lambda_\calK \right)
    &= \sum_{k} \Pr \left( \| p - \hat{p} \|_1 \geq \lambda_\calK \mid \calK = k \right) \Pr(\calK = k) \\
    &\leq \sum_{k} \left( 2^{S+1} e^{-2k \lambda_k^2} \right) \Pr(\calK = k),
\end{align*}

where the last inequality holds since for a fixed $\calK=k$, we have a bound which is not a random variable, but rather a number, hence, we can use \Cref{lem: Bretagnolle Huber Carol}.

We take the logarithm:
\begin{align*}
    &\log \operatorname{supp}(\calK) + (S+1)\log 2 - \sum_{k} 2k \lambda_k^2 \Pr(\calK = k)
    \\& = \log \operatorname{supp}(\calK) + (S+1)\log 2 - \sum_{k} 2k \left( \frac{\log(\operatorname{supp}(\calK)/\delta) + 2S}{2k} \right) \Pr(\calK = k).
    \\& = \log \operatorname{supp}(\calK) + (S+1)\log 2 - \left( \log(\operatorname{supp}(\calK)/\delta) + 2S \right)
    \quad \text{(since $\sum_{k} \Pr(\calK = k) = 1$)}
    \\
    &\leq \log \operatorname{supp}(\calK) + 2S - \log \operatorname{supp}(\calK) + \log \delta - 2S \\
    &= \log \delta.
\end{align*}

Hence,
\[
\Pr (\normone{p - \hat{p}} \geq \lambda_\calK) \leq \delta.
\]

\end{proof}

\begin{lemma}[Chernoff--Hoeffding, Lemma 10.2.1 in \cite{MannorMT-RLbook}]
\label{lem: Chernoff Hoeffding}
Let $R_1, \ldots, R_m$ be $m$ i.i.d.\ samples of a random variable $R \in [0,1]$.
Let $\mu = \mathbb{E}[R]$ and $\hat{\mu} = \frac{1}{m}\sum_{i=1}^m R_i$.
For any $\varepsilon \in (0,1)$, we have
\[
\Pr\left[ \lvert \mu - \hat{\mu} \rvert \ge \varepsilon \right]
   \le 2 e^{-2 \varepsilon^2 m}.
\]

\medskip
\noindent
In addition,
\[
\Pr\!\left[\hat{\mu} \le (1-\varepsilon)\mu\right] \le e^{-\varepsilon^2 \mu m / 2}
\qquad\text{and}\qquad
\Pr\!\left[\hat{\mu} \ge (1+\varepsilon)\mu\right] \le e^{-\varepsilon^2 \mu m / 3}.
\]

We will refer to the first bound as \emph{additive} and the second set of bounds as
\emph{multiplicative}.
\end{lemma}

\begin{lemma}[Based on Theorem 10.9 in \cite{MannorMT-RLbook}]
\label{lem: value function as occupancy measure distance}
Let $q_1$ and $q_2$ be two distributions over $\mathcal{S}$. Let
$f : \mathcal{S}\cup \{\sinkState\} \to [0, F_{\max}]$ and $f(\sinkState) = 0$.
Then,
\[
\left| \mathbb{E}_{s \sim q_1}[f(s)] - \mathbb{E}_{s \sim q_2}[f(s)] \right|
\leq F_{\max} \, \lVert q_1 - q_2 \rVert_1,
\]
where the norm is just over $\calS$:
\[
\lVert q_1 - q_2 \rVert_1 = \sum_{s \in \mathcal{S}} \left| q_1(s) - q_2(s) \right|.
\]
\end{lemma}

\begin{proof}
Based on \cite{MannorMT-RLbook},
\begin{align*}
\left| \mathbb{E}_{s \sim q_1}[f(s)] - \mathbb{E}_{s \sim q_2}[f(s)] \right|
&= \left| \sum_{s \in \mathcal{S}\cup \{\sinkState\}} f(s) q_1(s)
      - \sum_{s \in \mathcal{S}} f(s) q_2(s) \right| \\
&= \left| \sum_{s \in \mathcal{S}} f(s) q_1(s)
      - \sum_{s \in \mathcal{S}} f(s) q_2(s) \right| \quad \text{(since $f(\sinkState) =0$)}
\\ &\leq \sum_{s \in \mathcal{S}} f(s) \left| q_1(s) - q_2(s) \right| \\
&\leq F_{\max} \, \lVert q_1 - q_2 \rVert_1 .
\end{align*}
\end{proof}

\afterappendix

\end{document}